\documentclass[twoside, 11pt]{article}

\usepackage[abbrvbib, preprint]{jmlr2e}
\usepackage[T1]{fontenc}
\usepackage{amsmath}
  \allowdisplaybreaks[2]
\usepackage{xcolor}
\usepackage{graphicx}
\usepackage{tikzexternal}
\usepackage{paralist}
\usepackage{algorithm}
\usepackage{algpseudocode}

\title{Exclusive Group Lasso for Structured Variable Selection}
\author{\name David Gregoratti \email david.gregoratti@ieee.org \\
    \addr Formerly with the Centre Tecnol\`ogic de Telecomunicacions de
    Catalunya (CTTC/CERCA)\\
    Av.\ C.\ F.\ Gauss, 7, 08860 Castelldefels (Barcelona, Spain)
    \AND
    \name Xavier Mestre \email xavier.mestre@cttc.cat \\
    \addr Centre Tecnol\`ogic de Telecomunicacions de
    Catalunya (CTTC/CERCA)\\
    Av.\ C.\ F.\ Gauss, 7, 08860 Castelldefels (Barcelona, Spain)
    \AND
    \name Carlos Buelga\\
    \addr Formerly with the Centre Tecnol\`ogic de Telecomunicacions de
    Catalunya (CTTC/CERCA)\\
    Av.\ C.\ F.\ Gauss, 7, 08860 Castelldefels (Barcelona, Spain)
  }

\newcommand{\partition}{\mathcal{G}}
\newcommand{\support}{\mathcal{J}}
\newcommand{\supportnext}{\mathcal{K}}
\newcommand{\supportgeneric}{\mathcal{I}}
\newcommand{\tc}{\mathrm{c}}
\newcommand{\tinact}{\mathrm{inact.}}
\newcommand{\toff}{\mathrm{off}}
\DeclareMathOperator{\soft}{\mathcal{S}}
\DeclareMathOperator{\prox}{\mathrm{prox}}
\DeclareMathOperator{\signsupp}{\mathbb{S}_{\pm}}
\newcommand{\ord}[1]{\langle#1\rangle_G}

\newcommand{\md}[1]{\mathbf{#1}}
\newcommand{\ournorm}{\Omega_{\mathrm{excl}}}
\newcommand{\ournormrest}{\Omega_{\mathrm{excl},\support}}

\newcommand{\ournormrestoff}{\Omega_{\mathrm{excl},{\support^\toff}}}
\newcommand{\ournormrestinact}{\Omega_{\mathrm{excl},{\support^\tinact}}}
\newcommand{\ournormrestgeneric}{\Omega_{\mathrm{excl},\supportgeneric}}
\DeclareMathOperator{\supp}{supp}
\providecommand{\norm}[1]{\lVert#1\rVert}
\providecommand{\bignorm}[1]{\bigl\lVert#1\bigr\rVert}
\providecommand{\Bignorm}[1]{\Bigl\lVert#1\Bigr\rVert}
\providecommand{\biggnorm}[1]{\biggl\lVert#1\biggr\rVert}

\providecommand{\abs}[1]{\lvert#1\rvert}
\providecommand{\bigabs}[1]{\bigl\lvert#1\bigr\rvert}

\DeclareMathOperator*{\minimize}{minimize\,}
\newcommand{\tT}{\mathrm{T}}
\DeclareMathOperator{\sign}{sign}
\newcommand{\bm}{\mathbf}
\newcommand{\bs}{\boldsymbol}

\newcommand{\fig}{Figure}
\newcommand{\inputtikz}[1]{\tikz{}}

\tikzsetexternalprefix{figs/}
\begin{document}
\maketitle

\begin{abstract}%
A structured variable selection problem is considered in which the covariates,
divided into predefined groups, activate according to sparse patterns with few
nonzero entries per group. Capitalizing on the concept of \emph{atomic norm}, a
composite norm can be properly designed to promote such \emph{exclusive group
sparsity} patterns. The resulting norm lends itself to efficient and flexible
regularized optimization algorithms for support recovery, like the proximal
algorithm. Moreover, an active set algorithm is proposed that builds the
solution by successively including structure atoms into the estimated support.
It is also shown that such an algorithm can be tailored to match more rigid
structures than plain exclusive group sparsity. Asymptotic consistency analysis
(with both the number of parameters as well as the number of groups growing with
the observation size) establishes the effectiveness of the proposed solution in
terms of signed support recovery under conventional assumptions. Finally, a set
of numerical simulations further corroborates the results.
\end{abstract}


\section{Introduction}
The problem of estimating a sparse signal that has some known structural properties has arisen a lot of interest in statistical inference over the past decade. The problem was originally focused on recovering certain data or signal by exploiting a given parsimonious representation that describes it \cite{Candes06,Donoho06}. The underlying idea is that the observation can be modeled as a sparse linear combination of the columns of a certain measurement matrix, a fact that can be exploited in order to retrieve the signal with a number of measurements that is essentially commensurate with the sparsity level of the model rather than the number of free parameters. 

This basic sparse representation can be generalized to include more elaborate descriptions of the data, which have traditionally been referred to as structured sparsity. According to these models, the observation can be described as a linear combination of very few columns of a measurement matrix, where the activation of a certain covariate is inherently coupled with the activation of some additional ones. The most relevant example of this type of inherent signal structure is strong group sparsity \cite{Huang10}, where groups of covariates are simultaneously activated in the underlying linear model. 

Recently, there has been an increased interest in other types of sparsity structure which have been shown to be a more accurate description of data in real applications. This is the case of exclusive group sparsity, which accounts for the fact that the covariates in the underlying model can be collected in groups, so that only few covariates per group are in practice activated. This type of structure is more general than conventional strong group sparsity and therefore more difficult to exploit. Note, in fact, that exclusive group sparsity becomes conventional group sparsity when the activation of a certain covariate in a group is linked to the activation of some fixed covariates in the other ones. This higher degree of abstraction of exclusive group sparsity models has recently spurred its application in a number of fields, as diverse as multi-label image classification \cite{Chen11}, object tracking \cite{Zhang16}, behavioral research \cite{Kok19}, parallel acquisition \cite{Chun17}, MIMO radar \cite{Dorsch17}, gene selection or nuclear magnetic resonance spectroscopy \cite{Cambell17}.

In the recent literature, exclusive group sparsity has also been referred to as ``sparsity in levels'' \cite{adcock17}. Lately, there has been considerable interest in the link between exclusive group sparsity and multi-level sampling methods based on isometries. In particular, \cite{adcock17} provided some non-uniform recovery guarantees of exclusive group sparse signals by $\ell_1$ norm minimization in combination with multi-label sampling of an isometry measurement matrix. Later, in \cite{Li19} and \cite{Bastounis17} these recovery bounds were generalized to uniform guarantees by imposing a certain restricted isometry property in levels of the underlying matrix.  

Additional efforts have recently been focused on improved recovery methods for exclusive group sparsity problems. In particular, several recent works have focused on the use of optimization penalties that mimic the effect of classical Lasso \cite{Tibshirani96} and group Lasso \cite{Yuan2006} regularizers. The equivalent problem in exclusive group sparsity models is usually referred to as exclusive Lasso, and consists of a penalty first introduced in \cite{zhou10a} in the context of multi-task feature selection. This regularizer is in fact a composite norm, where the $\ell_1$ norm is applied group-wise, the outcome of which is then combined according to a conventional $\ell_2$ norm. In this sense, the exclusive group Lasso penalty can be seen as a direct transposition of conventional group Lasso, where the $\ell_2$ norm is applied to the different columns, which are then combined according to an $\ell_1$ paradigm. In exclusive group Lasso these norms are applied in reverse order. 

One of the most extensive studies of the exclusive group sparse regularizer was presented in \cite{Cambell17}. This paper considered this type of regularization in linear regression problems and studied the statistical consistency properties of the associated estimators. The fact that the exclusive Lasso penalty is not continuously differentiable (due to the presence of the $\ell_1$ norm at the group level) has motivated the study of optimized solutions to the associated optimization problem. For example, \cite{Cambell17} proposed a coordinate descent method, whereas \cite{Kong14} and \cite{yamada17a} considered the use of iteratively re-weighted algorithms to solve the penalized problem. Recently, \cite{lin2019dual} proposed a dual Newton based preconditioned proximal point algorithm for a weighted version of the exclusive Lasso.

It should be stressed here that the exclusive group sparse regularizer can be seen as a special case of more general mixed norm regularizers, where the $\ell_p$ norm is applied at the group level and the $\ell_q$ norm is used to combine the result. \cite{Kowalski2009} considered the use of these general mixed norm regularizers in regression problems and provided an expression for the proximal of the square of the weighted exclusive sparsity norm, which was then used as the basis of a thresholded Landweber algorithm. This can be used to formulate direct proximal gradient methods for the exclusive group Lasso, see also \cite{lin2019dual}. 

The objective of this paper is to study the application of exclusive group sparse penalty in order to promote certain sparsity structures beyond covariance exclusivity. To that effect, we will propose an active set algorithm based on the use of exclusive Lasso for agglomerative  activation of covariates according to some promoted sparsity structure. We will also extend the work in \cite{Cambell17} by establishing sharp recovery bounds for signed support recovery. To that effect we will consider the signal recovery via linear regression with exclusive sparsity norm penalization and follow the approach originally established in \cite{Wainwright09}. 

The rest of the paper is organized as follows. Section \ref{sec:formulation} introduces the exclusive group sparsity norm from the mathematical perspective and compares it with some other penalties in the literature that promote structured sparsity. Section \ref{sec:properties} presents some useful properties of the exclusive group sparsity norm that will be used throughout the paper. For minimization problems that include the exclusive group sparsity norm as a regularizer, Section~\ref{sec:proximal} proposes an efficient solving algorithm based on the proximal operator. Section \ref{sec:active_set} introduces an active set algorithm that is able to promote some sparsity structures along a certain evolution path. Section \ref{sec:consistency} studies the conditions that guarantee signed support consistency of the corresponding exclusive Lasso algorithm. Section \ref{sec:numresults} provides a comparative numerical assessment of the proposed algorithm and finally Section \ref{sec:conclusions} concludes the paper. Most of the technical proofs can be found in the Appendix.

\section{Exclusive group sparsity formulation} \label{sec:formulation}

Our objective is the recovery of a certain $p$-dimensional column vector $\md x$ that presents an exclusive group sparsity structure. The entries of $\md x$ will indistinctively be referred to as covariates, variables or entries, and will be denoted as $x_i$, $i \in [p]$, where $[p] = \{1,\ldots,p\}$. We will define the support of $\md x$ as $\support = \supp\{\md x\} = \{i\in[p]: x_i\ne 0\}$. The complementary of the support will be written as $\support^\tc = [p] \backslash \support$.

Let us consider a general partition $\partition$ of the set of indices $[p] = \{1,2,\dots,
p\}$, that is $\bigcup_{G \in \partition} G= [p]$ and $G\cap G' = \emptyset$ for
all $G, G' \in \partition, G\ne G'$. For any $\md x \in \mathbb{R}^p$, $p\in \mathbb{N}$, the exclusive group sparsity norm is defined as 
\begin{equation}\label{eq:norm_def}
\ournorm(\md x) = \sqrt{\sum_{G\in \partition} \norm{\md x_G}_1^2}
\end{equation}
where $\norm{\cdot}_1$ denotes the $\ell_1$ norm and where $\md x_G$ is a $|G|$-dimensional subvector of $\md x$ with entries indexed by
$G$. Observe that, as explained above, this definition corresponds to a composite norm, where the $\ell_1$ norm is applied at the group level ($\norm{\md x_G}_1$, $G \in \partition$) whereas the $\ell_2$ norm combines the contributions from the different groups. The conventional group sparsity norm follows the same approach and can simply be recovered by interchanging the order in which these two norms are applied.
\begin{remark} \label{remark:notation}
 In some parts of the paper, we will need to build a $p$-dimensional vector with all zero entries except for those indexed by $G \in \partition$, which correspond to $\md x_G$. We will denote this $p$-dimensional vector as $\md x_{\{G\}}$.
\end{remark}
The objective of this paper is to explore
the properties of norm $\ournorm(\md x)$ and to derive some efficient algorithms
to tackle optimization problems where $\ournorm(\md x)$ regularizes the solution
in favor of a characteristic sparsity pattern. Specifically, we are interested
in solution vectors where the active entries are evenly distributed among the groups of partition $\partition$ while promoting the selection of a sparse number of variables within each group (hence the name exclusive group sparsity).

The exclusive group sparsity structure enforced by $\ournorm(\md x)$ becomes
apparent once noticing that our norm is the composition between the $\ell_1$
norm (within each group) and the $\ell_2$ norm (among groups). This composite norm can also be seen as the \emph{atomic norm} (see \cite{Chandra2010}) induced by the set of atoms
$$
 \mathcal{A} = \Bigl\{\md a\in \mathbb{R}^p : \md a_G = \alpha_G \md e_i^{\abs{G}},
 \quad i\in[\abs{G}],\quad \alpha_G \in \mathbb{R}, \quad \forall G\in
 \partition, \quad\text{s.to }\sum_{G\in \partition}\alpha_G^2 = 1\Bigr\}
$$
where $\{\md e_i^{\abs{G}}\}_{i=1,2,\dots,\abs{G}}$ denotes the canonical basis of
 $\mathbb{R}^{\abs{G}}$. More formally,
 \begin{equation}\label{eq:norm_atoms}
 \ournorm(\md x) = \inf \Biggl\{ \sum_{\md a \in \mathcal{A}} c_{\md a} : \md x =
 \sum_{\md a \in \mathcal{A}} c_{\md a} \md a,\quad c_{\md a} \ge 0 \Biggr\}.
 \end{equation}
Here, we focus on practical aspects that arise when facing
the regularized optimization problems
\begin{equation}\label{eq:problems}
\minimize_{\md x\in \mathbb{R}^p} L(\md x) + \lambda \ournorm(\md x)
\quad\text{or}\quad
\minimize_{\md x\in \mathbb{R}^p} L(\md x) + \frac{\mu}{2} \ournorm^2(\md x)
\end{equation}
where $L: \mathbb{R}^p \to \mathbb{R}_+$ is a convex and continuously
differentiable loss function. It is worth recalling that the two problems above
can always be made equivalent by properly choosing the values of the two
regularization parameters $\lambda, \mu \in \mathbb{R}_+$, see \cite{Bach12}.

\subsection{Other approaches to structured sparsity}
A noticeable duality exists between the norm in (\ref{eq:norm_def}) and the
group-Lasso penalty $\sum_{G\in\partition} \norm{\md x_G}_2$, which also
consists in the
composition of the $\ell_1$ and $\ell_2$ norms but in the opposite order to
$\ournorm(\md x)$ (see, e.g., \cite{Yuan2006}). As one may expect, the effect is
profoundly different: In the group-Lasso penalty, the group is the structural
atom and, as such, all its elements are either active or set to zero
(\cite{Huang11}). Conversely, as we will prove in the next sections, penalty
$\ournorm(\md x)$ ensures that all groups are active and contain a comparable
number of nonzero regression coefficients.

\cite{Obozinski2011} extends the group-Lasso approach by applying the
$\ell_1/\ell_2$ norm to groups that may overlap (that is, they do not form a
partition of $[p]$ any more, even though all entries of $\md x$ belong to at
least one group). Their approach is similar to the one developed here, in the
sense that both approaches can be formulated as a minimum-weight atomic
representation analogous to the Minkowski functional.
Specifically, let $\mathcal{H}$ denote the family of overlapping groups and
$\md a_H$ be the generic vector supported by $H\in\mathcal{H}$ (the latent vectors
in the nomenclature of that paper). Then, the \emph{latent group Lasso} penalty
is defined as
\begin{equation} \label{eq:latentgrupLASSO}
\Omega_{\textrm{latent}}(\md x) = \inf_{\{\md a_H, c_H\}_{H\in\mathcal{H}}} \Biggl\{ \sum_{H \in \mathcal{H}} c_H :
  \md x = \sum_{H \in \mathcal{H}} c_{H} \md a_H,\quad c_{H} \ge 0,
  \quad \norm{\md a_H}_2 = d_H^{-1}\Biggr\}.
\end{equation}
where $d_H > 0$ is a fixed weight assigned to group $H\in\mathcal{H}$. The benefit
of latent group Lasso over classic group Lasso is the possibility to consider
sparsity patterns with more flexible structure than well-localized groups of
correlated covariates.

Interestingly enough, the affinity between the latent group Lasso and our
penalty (\ref{eq:norm_def}) is not limited to the underlying atomic formulation.
Indeed, it is sometimes possible to define the families of groups $\partition$,
associated to the exclusive group sparse penalty, and $\mathcal{H}$, for the latent
group Lasso, so that the induced sparsity structures are almost equivalent. For
instance, in \cite[Section 6.3]{Obozinski2011}, it is shown that the latent
group Lasso norm in the case $p=4$ and $\mathcal{H} = \{\{1,2\}, \{1, 3\}, \{2, 4\},
\{3, 4\}\}$ is equivalent to the exclusive group sparsity norm with groups $\partition =
\{\{1, 4\}, \{2, 3\}\}$, namely
$$
\Omega_{\textrm{latent}}(\md x) = \ournorm(\md x) = \sqrt{(\abs{x_1} +
\abs{x_4})^2 + (\abs{x_2} + \abs{x_3})^2}.
$$
Even if it is usually not possible to draw an exact equivalence between the two
norms as in this example, one can sometimes map the latent groups $\mathcal{H}$ into
a family $\partition$ that induces a similar, though slightly relaxed, sparsity
structure.
It is worth noticing that, conversely to the exclusive group sparsity norm
in~(\ref{eq:norm_def}), the latent group Lasso penalty usually lacks a
closed-form expression.

Several other works deal with structured sparsity, although there is no direct
correspondence between their models and ours (see, e.g.,
\cite{Shervashidze15,Rao16,Bayram17}).
A special mention is
deserved to \cite{Jenatton11b}, since it inspired the active set algorithm
derived in Section~\ref{sec:active_set}. In that paper the focus is still on
overlapping groups, with the considered norm being
\begin{equation} \label{eq:svspenalty}
\Omega_{\textrm{svs}}(\md x) = \sum_{H\in \mathcal{H}} \norm{\md d_H \odot \md
x}_2
\end{equation}
where $\md d_H$ is a vector supported by $H$ of positive weights and $\odot$
denotes the Hadamard element-wise product (and ``svs'' stands for
\emph{structured variable selection}). The main characteristic of the structured variable selection approach is that it
induces a solution support that is the intersection of a subset of the groups
contained in $\mathcal{H}$. Conversely, the latent group Lasso in (\ref{eq:latentgrupLASSO}) promotes solutions whose support is the
union of a subset of groups in $\mathcal{H}$. The outcome of both approaches is fundamentally different from the support patterns that are favored by the exclusive group sparse penalty, which promotes the activation of very few elements within each group of $\partition$, while guaranteeing that their contribution is evenly distributed across all groups. 
As we
will see in Section~\ref{sec:numresults}, this fact has a significant impact on the
support structures that can be detected.

\section{Norm proprieties} \label{sec:properties}
In this section we provide some basic properties of norm $\ournorm(\md x)$ which
will be useful in the rest of the paper. Besides, they offer a different
perspective for understanding the behavior of the norm as a regularizer and
for grasping more insight into the promoted sparsity structure.

\subsection{Dual norm and subdifferential}
To begin with, we give a closed-form expression for the dual norm associated to
$\ournorm(\md x)$ that, by definition, is given by $\ournorm^*(\md u) = \sup\{\md
u^\tT \md x: \md u \in \mathbb{R}^p, \ournorm(\md x) \le 1\}$. Routine computation
shows that, for the norm in (\ref{eq:norm_def}), the dual norm is
\begin{equation}\label{eq:dual_norm}
\ournorm^*(\md u) = \sqrt{\sum_{G\in\partition} \norm{\md x_G}_\infty^2}
\end{equation}
where we recall that $\mathcal{G}$ is a partition of $[p]$. 

Knowing the dual norm allows the following convenient characterization of the
subdifferential of the norm, that is the set of all subgradients of $\ournorm(\cdot)$
at $\md x  \in \mathbb{R}^p$ (see, e.g., \cite[Remark 1.1]{Bach12}):
\begin{align}
\partial \ournorm(\md x)
&= \{\md u\in \mathbb{R}^p: \ournorm(\md y) \ge \ournorm(\md x) +
  \md u^\tT (\md y - \md x), \forall \md y \in \mathbb{R}^p\} \nonumber \\
&= \{\md u\in \mathbb{R}^p: \md u^\tT \md x = \ournorm(\md x) \text{ and }
    \ournorm^*(\md u) \le 1\}. \label{eq:subdiff}
\end{align}
Indeed, the Fenchel--Young inequality states that, for $\md u, \md x \in \mathbb{R}^p$, we have $\md x^\tT \md u \leq \ournorm(\md x) + \mathbb{I}_{\ournorm^\ast(\md u) \leq 1}$, where $\mathbb{I}_{A}$ is the identity function associated to the event $A$ ($\mathbb{I}_{A}=0$ when $A$ is true and $\mathbb{I}_{A}=+\infty$ otherwise). Furthermore, strict equality holds if and only if $\md u \in \partial\ournorm(\md x)$. We can therefore identify the elements of $\partial\ournorm(\md x)$ with the elements $\md u \in \mathbb{R}^p$ for which Fenchel--Young holds with equality, that is (\ref{eq:subdiff}).

The following characterization of the subdifferential will be helpful in
deriving the results of the next sections.
\begin{lemma}\label{lemma:norm_dual}
For all $\md x \in \mathbb{R}^p\setminus\{\md 0\}$, $\md u \in \partial
\ournorm(\md x)$ if and only if
\begin{align*}
& \abs{u_i} = \frac{\norm{\md x_{G_i}}_1}{\ournorm(\md x)}
\qquad\text{and}\qquad \sign(u_i) = \sign(x_i) &&\forall i: x_i \ne 0\\
&\hspace{3cm}\abs{u_i} \le \frac{\norm{\md x_{G_i}}_1}{\ournorm(\md x)}
&&\forall i: x_i = 0
\end{align*}
where $G_i$ is the (unique) group containing index $i$.
\end{lemma}
\begin{proof}
Sufficiency is direct once we observe that the conditions on $\abs{u_i}$ imply
that
$$
\norm{\md u_G}_\infty = \frac{\norm{\md x_G}_1}{\ournorm(\md x)}
$$
for all $G\in \partition$.

To prove necessity, we can proceed as follows.
By (\ref{eq:subdiff}) and the definition of dual norm, we have
$$
\ournorm(\md x) = \md u^\tT \md x \le \ournorm(\md x) \ournorm^*(\md u) \le
\ournorm(\md x)
$$
which implies that $\md u \in \partial \ournorm(\md x)$, $\md x \ne \md 0$, if
and only if equality holds everywhere. But,
$$
\md u^\tT \md x = \sum_{G\in \partition} \md u_G^\tT \md x_G
\stackrel{\mathrm{(a)}}{\le} \sum_{G\in \partition} \norm{\md u_G}_\infty
  \norm{\md x_G}_1
\stackrel{\mathrm{(b)}}{\le} \sqrt{\sum_{G\in \partition} \norm{\md u_G}_\infty^2}
  \sqrt{\sum_{G\in \partition} \norm{\md x_G}_1^2}
= \ournorm(\md x) \ournorm^*(\md u)
$$
and, by the previous observation, equality must hold everywhere.

An equality in (a) implies that $\abs{u_i} = \norm{\md u_G}_\infty$ and
$\sign(u_i) = \sign(x_i)$ for all
$i\in G$ such that $x_i \ne 0$, as well as $\abs{u_i} \le \norm{\md u_G}_\infty$
for all $i\in G$ such that $x_i = 0$. Also, (b)~follows from the Cauchy--Schwarz
inequality and holds with equality if and only if $\norm{\md u_G}_\infty =
\alpha \norm{\md x_G}_1$ for some $\alpha > 0$ and for all $G\in \partition$. In
summary, if $\md u \in \partial \ournorm(\md x)$ then, for all $G$, $\md u_G$
has constant amplitude on the support of $\md x_G$ and
$$
1 = \bigl(\ournorm^*(\md u)\bigr)^2
= \sum_{G\in \partition} \norm{\md u_G}_\infty^2
= \alpha^2 \sum_{G\in \partition} \norm{\md x_G}_1^2
= \alpha^2 \ournorm^2(\md x).
$$
Whence, $\alpha = 1/\ournorm(\md x)$ and
$$
\norm{\md u_G}_\infty = \frac{\norm{\md x_G}_1}{\ournorm(\md x)}\qquad \forall G\in
\partition
$$
which yields the requirements on $\abs{u_i}$.
\end{proof}

\subsection{Variational formulations}
Next, we introduce an expression for norm $\ournorm(\md x)$ that is
alternative to the closed form in (\ref{eq:norm_def}) and to the gauge function
of the atom set in (\ref{eq:norm_atoms}), which may also be considered a
variational formulation.

\begin{lemma}[Second variational formulation]\label{lemma:variational2}
For any $\md x\in\mathbb{R}^p$, norm $\ournorm(\md x)$ can also be expressed in
terms of the following maximization problem:
\begin{equation}\label{eq:variational2}
\ournorm(\md x) = \max_{\md u\in \mathbb{R}^p, \{t_G\}} \md u^\tT \md x \quad
  \text{s.to } u_i^2 \le t_G\quad \forall i\in G, \forall G\in \partition \text{ and }
  \sum_{G\in \partition} t_G \le 1.
\end{equation}
\end{lemma}

\begin{proof}
This result proceeds directly from writing $\ournorm(\md x)$ as the dual norm
of its dual norm~(\ref{eq:dual_norm})
and casting the underlying maximization problem as a conic program:
\begin{align*}
\ournorm(\md x) &= \max_{\md u \in \mathbb{R}^p} \md u^\tT \md x\quad\text{s.to }
  \ournorm^*(\md u) \le 1 \\
&= \max_{\md u \in \mathbb{R}^p} \md u^\tT \md x\quad\text{s.to }
  \sum_{G\in\partition} \norm{\md u_G}_\infty^2 \le 1 \\
&= \max_{\md u \in \mathbb{R}^p, \{t_G\}} \md u^\tT \md x\quad\text{s.to }
  \sum_{G\in\partition} t_G \le 1\text{ and }\norm{\md u_G}_\infty^2 \le t_G,
  \forall G.
\end{align*}
\end{proof}
This formulation allows us to introduce the scalars $\{t_G\}_{G \in \mathcal{G}}$, which can be seen as an upper bound on the square of the infinity norm of the dual variables, that is $\|\md u_G\|_\infty^2$. 
By formulating the Karush-Khun-Tucker conditions of the problem in  (\ref{eq:variational2}), one readily sees that, inside any group $G$,
$u_i^2 = t_G$ whenever $x_i \ne 0$, $i \in G$, and $u_i^2 \le t_G$ otherwise.
Also, as a consequence of the above lemma, it is evident that the set of solutions of
(\ref{eq:variational2}) is given by the subdifferential $\partial \ournorm(\md x)$
in (\ref{eq:subdiff}).

\section{Proximal operator}\label{sec:proximal}
\begin{remark}
The algorithm described in this section was developed autonomously. However, we
later became aware that similar techniques had already been published, see e.g.\
\cite{lin2019dual}. For this reason, we decided not to include it in the
peer-reviewed version of the paper. In any case, we point out that previously
derived expressions for the proximal refer to the square of the exclusive 
sparsity norm, rather than the norm itself (which is the one derived here). 
\end{remark}
Having assessed the main properties of the exclusive group sparsity norm $\ournorm(\md x)$,
we now steer our focus towards the regularized optimization problems in
(\ref{eq:problems}), starting from
$$
\minimize_{\md x\in \mathbb{R}^p} L(\md x) + \lambda \ournorm(\md x).
$$

Proximal gradient methods (or simply proximal methods) are a very common and
successful approach to this type of problems, where the objective function is
composed by a convex differentiable term (i.e., the loss function $L(\md x)$)
and a convex but nonsmooth term (i.e., the penalty $\lambda \ournorm(\md x)$), as
indicated by the conspicuous number of related works (e.g., \cite{Villa2014,
Jenatton11}). This is especially true in high dimensional settings, where
other approaches like the interior-point method become too complex for practical
applications.

To be more specific, let us first recall the definition of the proximal operator
for norm $\ournorm(\md x)$, namely
\begin{equation}\label{eq:proximal}
\prox_{\ournorm}(\md x) = \arg\min_{\md z} \frac{1}{2}\norm{\md z - \md x}_2^2 +
\ournorm(\md z).
\end{equation}
which suggests a strong connection to the projection of $\md x$ onto a level set
of $\ournorm(\cdot)$ (i.e., $\arg\min_{\md z} \frac{1}{2} \norm{\md z - \md
x}_2^2$ subject to $\ournorm(\md z) \le \beta$ for some $\beta\ge 0$). Indeed, the
proximal method can be seen as an extension of the projected gradient method
and, as such, it shares most of the appealing features first introduced by
\cite{Nesterov83}. In particular, a convergence rate of $O(k^{-2})$, with $k$
the iteration index, can be proven for the iteration
\begin{subequations}\label{eq:fista}
\begin{align}
\md x^k &= \prox_{\frac{\lambda}{\gamma_L}\ournorm}\biggl(\md w^k -
\frac{1}{\gamma_L} \nabla L(\md x^{k-1})\biggr) \label{eq:fista1} \\
\xi^{k+1} &= \biggl(1+\sqrt{1+4(\xi^k)^2}\biggr) / 2 \label{eq:fista2} \\
\md w^{k+1} &= \md x^k + \frac{\xi^k-1}{\xi^{k+1}}(\md x^k - \md x^{k-1})
\end{align}
\end{subequations}
when $L(\md x)$ is Lipschitz continuous gradient with constant
$\gamma_L$ (see \cite{Beck09} or, for a
different yet equivalent step (\ref{eq:fista2}), \cite{Nesterov2007}). Note that
the proximal operator for the scaled norm in (\ref{eq:fista1}) is related
to the original one in (\ref{eq:proximal}) through the trivial identity
$\prox_{\lambda\ournorm}(\md x) = \lambda \prox_{\ournorm}(\lambda^{-1}\md x)$.

Given that all other operations in (\ref{eq:fista}) are pretty straightforward,
the computational bottleneck of the proximal method is the complexity of the
proximal operator itself. Fortunately, the exclusive group sparsity norm results in a
fairly simple proximal operator.

\begin{theorem}
For any $\md x\in\mathbb{R}^p$, the proximal operator of the norm defined by
(\ref{eq:norm_def}) maps $\md x$ to the vector $\prox_{\ournorm}(\md x)$ with
components
$$
(\prox_{\ournorm}(\md x))_G = \soft_{t_G}(\md x_G)
$$
for all $G\in\partition$, where $\soft_t(\md x)$ is the soft-thresholding
operator applied entry-wise, i.e.,
$$
(\soft_t(\md x))_i = \begin{cases}
\frac{x_i}{\abs{x_i}}(\abs{x_i} - t) & \text{if } \abs{x_i} \ge t \\
0 & \text{otherwise}
\end{cases}
$$
where thresholds $\{t_G\}_{G\in\partition}$ are given by
$$
t_G = \frac{1}{n_G + \eta} \sum_{i\in \bar{G}} \abs{x_i}
$$
and where $\bar{G} = \{i\in G: \abs{x_i}\ge t_G\}$, $n_G = \abs{\bar{G}}$ and
$\eta$ is a positive constant such that
\begin{equation}\label{eq:tot_constraint}
\sum_{G\in \partition} t_G^2 = \sum_{G\in\partition}
  \frac{\biggl(\sum_{i\in\bar{G}}\abs{x_i}\biggr)^2}{(n_G + \eta)^2} = 1.
\end{equation}
\end{theorem}

Before getting into the details of the proof, it is worth pointing out that, in
spite of the lack of a closed-form expression, thresholds $\{t_G\}_{G\in
\partition}$ can be easily computed by the \emph{waterfilling-like} procedure \cite[Example 5.2]{Boyd04} in Algorithm~\ref{alg:proximal}. 

\begin{algorithm}
\caption{Compute $\prox_{\ournorm}(\md x)$}\label{alg:proximal}
\begin{algorithmic}[1]
\State \textbf{Input:} vector $\md x\in\mathbb{R}^p$;
\State \textbf{Output:} vector $\prox_{\ournorm}(\md x)\in\mathbb{R}^p$;
\State \textbf{Initialize:} $\bar{G} \gets \{\arg\max_{i\in G} \{\abs{x_i}\}\}$ for all
$G\in \partition$;
\State compute $\eta$ from (\ref{eq:tot_constraint}) by means of, e.g., the
Newton--Raphson method;
\While{condition (\ref{eq:eta_bounds}) does not hold for all groups}
\State choose the group $G^*$ that, by activating the largest inactive
entry, corresponds to the ``change of piece'' in (\ref{eq:fG}) that is closest
to the current value of $\eta$
\State $\bar{G^*} \gets \bar{G^*} \cup \{\arg\max_{i\in G^*\setminus\bar{G^*}}
\abs{x_i}\}$
\State compute $\eta$ from (\ref{eq:tot_constraint}) by means of, e.g., the
Newton--Raphson method;
\EndWhile
\item $t_G \gets \frac{1}{n_G + \eta} \sum_{i\in \bar{G}} \abs{x_i}$ for all
$G\in \partition$;
\item $\prox_{\ournorm}(\md x) \gets (\soft_{t_G}(\md x_G))_{G\in\partition}$.
\end{algorithmic}
\end{algorithm}

\begin{proof}
Leveraging the well known result (see, e.g., \cite{Moreau62, Combettes11}),
$$
\prox_{\ournorm}(\md x) = \md x - \Pi_{\ournorm^*}(\md x)
$$
where $\Pi_{\ournorm^*}(\md x):\mathbb{R}^p \to \mathbb{R}^p$ is the
projection of $\md x$ onto the unitary ball of $\ournorm^*(\cdot)$, i.e.\ the dual norm
of $\ournorm(\cdot)$, the computation of the proximal operator reduces to solving
the optimization problem
$$
\Pi_{\ournorm^*}(\md x) = \arg\min_{\md z\in\mathbb{R}^p} \frac{1}{2}\norm{\md z - \md x}_2^2\quad
\text{s.to } \ournorm^*(\md z) \le 1.
$$

Plugging (\ref{eq:dual_norm}) into the previous identity and introducing the
auxiliary variables $\{t_G\}_{G\in\partition}$, the projection can be
equivalently written as
$$
\Pi_{\ournorm^*}(\md x) = \arg_{\md z}\min_{\md z\in\mathbb{R}^p, \{t_G\}_{G\in\partition}}
\frac{1}{2}\norm{\md z - \md x}_2^2\quad
\text{s.to } \sum_{G\in \partition} t_G^2 \le 1 \text{ and } \norm{\md
z_G}_{\infty} \le t_G, \forall G\in \partition.
$$
The minimization with respect to $\bm z$ yields
$$
(\Pi_{\ournorm^*}(\md x))_G = \md x_G - \soft_{t_G^*}(\md x_G)
$$
where
\begin{align}
\{t_G^*\}_{G\in\partition} &= \arg\min_{\{t_G\}_{G\in\partition}}
\frac{1}{2}\sum_{G\in\partition}\norm{\soft_{t_G}(\md x_G)}_2^2\quad
\text{s.to } \sum_{G\in \partition} t_G^2 \le 1 \nonumber \\
&= \arg\min_{\{t_G\}_{G\in\partition}}
\frac{1}{2}\sum_{G\in\partition} \sum_{i\in G: \abs{x_i} \ge t_G} (\abs{x_i} -
t_G)^2 \quad\text{s.to } \sum_{G\in \partition} t_G^2 \le 1.
\label{eq:proj_equiv}
\end{align}

Let $\ord{i}:[\abs{G}]\to G$ be a bijection such that $\abs{x_{\ord{1}}} \ge
\abs{x_{\ord{2}}} \ge \dots \ge \abs{x_{\ord{|G|}}}$. Then, the Lagrangian of
the above minimization problem reads
$$
\mathcal{L}(\{t_G\}, \eta) = \frac{1}{2} \sum_{G\in\partition} \sum_{i:
\abs{x_{\ord{i}}}\ge t_G} (\abs{x_{\ord{i}}}-t_G)^2 +
\frac{\eta}{2}\Biggl(\sum_{G\in\partition} t_G^2 - 1\Biggr).
$$
Taking derivatives with respect to $t_G$ we obtain
\begin{equation}\label{eq:lagrangian}
\frac{\partial\mathcal{L}}{\partial t_G} = \begin{cases}
\eta t_G & \text{if } t_G > \abs{x_{\ord{1}}} \\
\sum_{i=1}^{n} (t_G - \abs{x_{\ord{i}}}) + \eta t_G & \text{if } \abs{x_{\ord{n+1}}} < t_G \le
\abs{x_{\ord{n}}}
\end{cases}
\end{equation}
for $n=1,2,\dots,\abs{G}$ and where we have introduced the definition $x_{\ord{\abs{G}+1}} = 0$. Note that
$\mathcal{L}(\{t_G\}, \eta)$ is piecewise-defined, continuous and
differentiable (right and left derivatives at $t_G = \abs{x_{\ord{i}}}$ coincide) with respect
to any of the variables $t_G$.

By examining the Karush-Khun-Tucker conditions of the above problem we can reach the conclusion
that either $\eta=0$ and $\sum_{G\in\partition}\norm{\md x_G}_\infty^2 < 1$, or
$\eta > 0$ and $t^\ast_G \le \norm{\md x_G}_\infty$ for all $G\in\partition$. 
Indeed, consider first the case $\eta =0$ and assume that there exists an integer $n_G, 1 \leq n_G \leq |G|$,
such that $\abs{x_{\ord{n_G+1}}} < t^\ast_G \le \abs{x_{\ord{n_G}}}$. Using the 
stationarity conditions obtained by forcing (\ref{eq:lagrangian}) to zero, we know that $t^\ast _G = \frac{1}{n_G} \sum_{i=1}^{n_G}|x_{\ord{i}}| \geq |x_{\ord{n_G}}|$. But since 
$t^\ast_G \leq |x_{\ord{n_G}}|$ by assumption, we see that one must have
$t^\ast_G=\abs{x_{\ord{n_G}}}= \ldots = \abs{x_{\ord{1}}} = \| \md x_{G}\|_\infty$.
When it is not possible to select such an $n_G$, we will have $t^\ast_G > |x_{\ord{1}}| = \|x_G\|_\infty$, 
so that either way we can write $t^\ast_G \geq \|x_G\|_\infty$ for all $G \in\partition$.
Now, by the complementary slackness condition we see that  $\sum_{G\in\partition}\norm{\md x_G}_\infty^2 \leq \sum_{G\in\partition}(t^\ast_G)^2 < 1$ as we wanted to show. 

Consider now the case $\eta>0$. In this
second situation, we focus again on the stationarity conditions obtained by forcing (\ref{eq:lagrangian}) to zero. The case $t^\ast_G > \|\md x_G\|_\infty$ would imply 
$t^\ast_G = 0$, meaning that necessarily $t^\ast_G = \|\md x_G\|_\infty = 0$, a
contradiction. Therefore, one must always have $t^\ast_G \leq \|\md x_G\|_\infty$ if $\eta >0$. Consider the situation where $\abs{x_{\ord{n_G+1}}} < t^\ast_G \le \abs{x_{\ord{n_G}}}$, implying that 
$$
t_G^* = \frac{1}{n_G + \eta}\sum_{i=1}^{n_G} \abs{x_{\ord{i}}}
$$
where here again $n_G$ denotes the number of active entries in group $G$, namely $n_G = \abs{\{i\in G: \abs{x_i} \ge t^\ast_G\}}$ for all $G$. 
Note that, because of
(\ref{eq:lagrangian}), we must have
$$
\abs{x_{\ord{n_G+1}}} < \frac{1}{n_G + \eta} \sum_{i=1}^{n_G} \abs{x_{\ord{i}}} \le
\abs{x_{\ord{n_G}}}
$$
or, equivalently,
\begin{equation}\label{eq:eta_bounds}
\frac{\sum_{i=1}^{n_G-1} \abs{x_{\ord{i}}}}{\abs{x_{\ord{n_G}}}} - (n_G-1)\le
\eta < \frac{\sum_{i=1}^{n_G} \abs{x_{\ord{i}}}}{\abs{x_{\ord{n_G+1}}}} - n_G
\end{equation}
which confirms that $\eta>0$ and, thus, $\eta$ is dual feasible.
Moreover, primal feasibility and complementary
slackness imply that
$$
\sum_{G\in\partition} (t_G^*)^2 = \sum_{G\in\partition}
\frac{\biggl(\sum_{i=1}^{n_G}\abs{x_{\ord{i}}}\biggr)^2}{(n_G + \eta)^2} = 1.
$$

Finally, let us consider the following family of piecewise-defined functions $\{f_G(\cdot)\}$ as
\begin{equation}\label{eq:fG}
f_G(\eta) =
\frac{\biggl(\sum_{i=1}^{k}\abs{x_{\ord{i}}}\biggr)^2}{(k + \eta)^2}\qquad
\text{if }
\frac{\sum_{i=1}^{k-1} \abs{x_{\ord{i}}}}{\abs{x_{\ord{k}}}} - (k-1)\le
\eta < \frac{\sum_{i=1}^{k} \abs{x_{\ord{i}}}}{\abs{x_{\ord{k+1}}}} - k
\end{equation}
for all $k=1,2,\dots,\abs{G}$.
One readily sees that each function $f_G(\eta)$ is continuous and decreasing in
$\eta$ in each definition interval. Then, $f_G(\eta)$ is decreasing in $\eta$
over $[0,+\infty)$, with $f_G(0)=\norm{\md x_G}_\infty^2$ and $f_G(\eta)\to 0$ as
$\eta\to\infty$. Since the sum of decreasing functions is a decreasing function
and since we are considering the case $\sum_{G\in\partition} \norm{\md
x_G}_\infty^2 \ge 1$, we can ensure that $\sum_{G\in\partition} f_G(\eta) = 1$ has
a unique solution and, in turn, that the minimum point of
problem~(\ref{eq:proj_equiv}) can be computed by Algorithm~\ref{alg:proximal}.
\end{proof}

\section{Active set}\label{sec:active_set}

Active set algorithms typically offer substantial computational advantages
in comparison to proximal methods. 
In an active set algorithm, the support of the estimate $\hat{\md x}$ is
recovered iteratively, starting from the empty set (that is $\hat{\md x} = \md
0$). At each successive step, the current support $\support$ is checked against
an optimality condition: If the condition is satisfied, the algorithm ends;
otherwise, new inactive variables are included in the support and a new
iteration is carried out. They constitute an extremely appealing option when
the support of the solution is expected to be fairly small. We present next an active set algorithm for the minimization problem with the
squared-norm regularizer, namely
\begin{equation}\label{eq:problem_norm2}
\minimize_{\md x\in \mathbb{R}^p} L(\md x) + \frac{\mu}{2} \ournorm^2(\md x).
\end{equation}

Note that the proposed active set method is a forward, greedy algorithm, meaning that variables may only enter the
active set and in no case are we allowed to change our mind and set an active
variable back to~0. Therefore, the evolution path must be strategically designed
so that it reflects the way the solution support actually grows. When, as in our
case, the regularizer
enforces a sparsity structure that admits
an ``atomic'' representation [see (\ref{eq:norm_atoms})], a natural approach
consists in assuming that the support expands one atom at a time, from the null
set to the entire $[p]$. In other words, the solution vector is built by
including, at each iteration of the active set algorithm, the atom that most
improves the quality of the solution in terms of duality gap, as explained next. The evolution of
the support throughout the algorithm can thus be represented by a \emph{directed
acyclic graph}, as the
one depicted in \fig~\ref{fig:evolution} for $p=6$ and groups $\partition=\{\{1,3,5\},\{2,4,6\}\}$. 

Let us explain a bit more the example in \fig~\ref{fig:evolution} in order to
illustrate the mechanics of the proposed active set method. The algorithm starts by
assuming $\mathcal{J} = \emptyset$ and $ \hat{\md x}=0$, which corresponds to the 
first node on the left. From this node, one can progress to the right by
following one of the outgoing edges into nine possible support updates. These
new supports are associated to the nodes that are depicted in the second
column of \fig~\ref{fig:evolution}. They correspond to potential solutions with exactly one active element in each of the
two groups of $\partition$, that is one active element in an even position and 
another active element in an odd position. Assume that in the first step the 
algorithm selects the node at the top of the second column, i.e. $\mathcal{J} =
\{5,6\}$ (see also the cyan arrows in \fig~\ref{fig:evolution}).
At the next step, these two entries are assumed active, and the algorithm considers
expansions of the support that require the activation of either one additional 
element of one of the two groups of $\partition$ (giving rise to the supports
$\{4,5,6\},\{2,5,6\},\{3,5,6\},\{1,5,6\}$) or two additional elements of
either one of these two groups (giving rise to the supports $\{3,4,5,6\},\{2,3,5,6\},\{1,4,5,6\},\{1,2,5,6\}$). The algorithm then 
proceeds by selecting the appropriate edge according to a criterion that will
be discussed below until some stopping condition is met. 

Interestingly, this interpretation also implies that we
can artificially modify the evolution graph (for instance, by allowing only a
subset of the evolution path at each step) to enforce sparsity structures that
are not specifically addressed by the regularizer alone. For instance, the
framed configurations in \fig~\ref{fig:evolution} correspond to the evolution
path of the latent group Lasso with overlapping groups $\mathcal{H}=\{\{1,2\},
\{2,3\},\{3,4\},\{4,5\},\{5,6\}\}$. This means that we can apply the proposed
algorithm in order to deal with the sparsity structures that are enhanced in
the latent group Lasso algorithm. We will further investigate
this approach by comparing the two algorithms in the numerical results section.

\begin{figure}
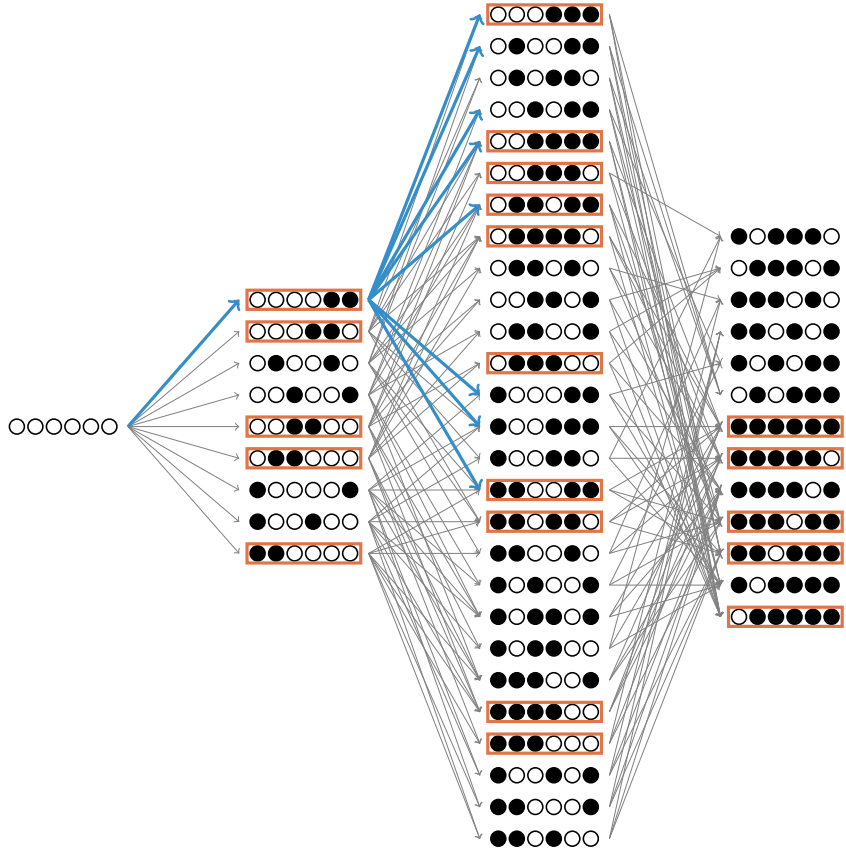

\centering
\inputtikz{./figs/single/dag}
\caption{Evolution path of $\ournorm(\cdot)$ for $p=6$ and
$\partition=\{\{1,3,5\}, \{2,4,6\}\}$. The cyan edges refer to the example in the
main text. The framed configurations correspond to
the evolution path of the latent group norm with groups $\mathcal{H}=\{\{1,2\},
\{2,3\},\{3,4\},\{4,5\},\{5,6\}\}$.} \label{fig:evolution}
\end{figure}

Now that we have a suitable evolution path, we only need to specify the optimality
rule that allows, at each iteration, growing the support so that the quality of
the solution improves. To that end, we propose a suboptimal approach that consists in a
two-step rule: a necessary condition and a sufficient condition. The algorithm
starts by checking the necessary condition only, which ensures that the current
support is included in the true one. Once the necessary condition is satisfied, new
variables are included in the solution support until we can ensure that the
duality gap is sufficiently small.

Before proceeding further, let us introduce some helpful notation.
For any set of indices $\mathcal{I}\subseteq [p]$, let $\partition_{\mathcal{I}}
\subseteq \partition$ be the family of groups $G$ with at least one element in
$\mathcal{I}$, that is
$$
\partition_{\mathcal{I}} = \{G \in \partition: G\cap\mathcal{I}\ne \emptyset\}.
$$
Also, let $\support^\tc$ be the complementary set of $\support$ in $[p]$.
The next proposition establishes a necessary property that a solution $\hat{\md x}$
to the problem in (\ref{eq:problem_norm2}) should have in terms of its support, 
which we denote as $\support$. The proof is provided in
Appendix~\ref{apdx:proof_active_set}.

\begin{proposition}[Necessary Condition]\label{prop:active_necessary}
If vector $\hat{\md x}$ with support $\support$ is an optimal point for
(\ref{eq:problem_norm2}), then necessarily
\begin{equation}\label{eq:active_nec_1}
\max_{\supportnext\in\Pi(\support)}
\max_{G\in\partition_{\support}\cap\partition_{\supportnext\setminus\support}}
\frac{\Bignorm{[\nabla L(\hat{\md x})]_{G\cap(\supportnext\setminus\support)}}_{\infty}}
{\norm{\hat{\md x}_{G\cap\support}}_1} \le \mu
\end{equation}
and
\begin{equation}\label{eq:active_nec_2}
[\nabla L(\hat{\md x})]_{G} = \md 0\qquad\forall G: G\cap \support = \emptyset
\end{equation}
where $\Pi(\support)$ is the set of possible active sets that can be reached
from $\support$ following the evolution path. 
\end{proposition}

The necessary condition in (\ref{eq:active_nec_2}) basically states that the gradient
of the loss function at the solution (i.e., $\nabla L(\hat{\md x})$) is zero at positions
corresponding to inactive groups (groups that do not intersect with the solution support).
On the other hand, the necessary condition in (\ref{eq:active_nec_1}) focuses on inactive
entries of active groups (i.e. entries not contained in the solution support $\support$
but still contained in a group with active entries). The gradient of the loss function at
the positions corresponding to these entries is, up to a constant $\mu$, upper bounded by 
the 1-norm of the active entries of that group.

The following proposition provides a sufficient condition for a particular $\hat{\md x}$, 
with support $\support$, to be sufficiently close to the solution of the problem in
(\ref{eq:problem_norm2}) in the sense that the duality gap is below a certain threshold. 
To that effect, we will assume that the $|\support|$-dimensional vector 
$\hat{\md x}_\support$ is the solution to the problem constrained to the support
$\support$ (see also Appendix~\ref{apdx:proof_general}) and establish conditions that
guarantee that this is close enough to the solution of the general problem, in the sense
that the duality gap is below a certain $\epsilon>0$. The proof is given in  Appendix~\ref{apdx:proof_active_set}. To formulate this result, we define $L_\support(\md x_\support)$ and $\ournormrest(\md x_\support)$ as the restriction of $L(\md x)$ and $\ournorm(\md x)$ to arguments with support in $\support$.

\begin{proposition}[Sufficient Condition]\label{prop:active_sufficient}
Let $\hat{\md x}$ be a vector of support $\support$ such that $\hat{\md x}_\support$ is the
solution to the reduced problem
$$
\minimize_{\md x_\support\in \mathbb{R}^{\abs{\support}}}
  L_\support(\md x_\support) + \frac{\mu}{2}
  \ournormrest^2(\md x_\support).
$$
Then, if
$$
\begin{cases}
\displaystyle
\sum_{G\in\partition_\support} \Bignorm{\bigl[\nabla
L(\hat{\md x})\bigr]_{G\cap\support^\tc}}_{\infty}^2 \le
  2\mu\epsilon + \mu^2 \ournormrest^2(\hat{\md x}_\support) \\
\displaystyle
\sum_{G\in(\partition_\support)^\tc} \Bignorm{\bigl[\nabla
L(\hat{\md x})\bigr]_{G}}_{\infty}^2 \le 
  2\mu\epsilon +  \mu^2 \ournormrest^2(\hat{\md x}_\support)
\end{cases}
$$
vector $\hat{\md x}$ is an approximate solution to the full problem
(\ref{eq:problem_norm2}) that achieves a duality gap not larger than
$\epsilon>0$.
\end{proposition}

The sufficient conditions are established on the magnitude of the gradient of the loss
function evaluated at inactive entries (i.e. entries outside the support $\support$). 
More specifically, the first condition is established on the magnitude of the gradient at
inactive entries of active groups, whereas the second is established on the magnitude of
the gradient at inactive groups (i.e. groups that do not intersect with the support). In
both cases, we consider the sum of squares of maximum magnitude of the gradient at the
inactive entries of each group. It is sufficient to establish that this sum is upper 
bounded by the squared exclusive group sparse penalty plus a small constant that is 
proportional to the duality gap associated to the intended solution $\hat{\md x}$.

Roughly speaking, the two conditions are supporting the idea that we can solve
a simpler, typically well-conditioned version of the problem with a known
support $\support$. This restricted solution $\hat{\md x}$ provides a good
approximation of the solution to original problem as long as the gradient of the
loss function, namely $\nabla L(\cdot)$, is well behaved at $\hat{\md x}$,
meaning that no entry is excessively large as compared to the magnitude of
$\hat{\md x}$. Algorithm~\ref{alg:active} applies this approach to determine the
support of the solution to problem (\ref{eq:problem_norm2}).

\begin{algorithm}
\caption{Active set algorithm.}\label{alg:active}
\begin{algorithmic}[1]
\State \textbf{Input:} Cost function $L$, regularization parameter $\mu$ and
  sparsity upper bound $s$;
\State \textbf{Output:} Active set $\support$ and approximated solution
  $\hat{\md x}$;
\State \textbf{Initialize:} $\support\gets \emptyset$, $\hat{\md x} \gets 0$;
\While{the necessary condition in Proposition~\ref{prop:active_necessary} fails
and $\abs{\support} < s$}
\State find the group that violates Proposition~\ref{prop:active_necessary} the most
  (either the $G\notin \partition_{\support}$ with the highest $\Bignorm{\bigl[\nabla
  L(\hat{\md x})\bigr]_G}_\infty$ or the $G\in\partition_{\support}$ with the largest product
  $\Bignorm{\bigl[\nabla L(\hat{\md x})\bigr]_{G\cap \support^\tc}}_\infty
  \norm{\hat{\md x}_{G\cap\support}}_1^{-1}$);
\State $\support \gets \support \cup \Bigl\{i=\arg\max_{i\in G} \bigabs{\bigl[\nabla
  L(\hat{\md x})\bigr]_i}\Bigr\}$;
\State compute $\hat{\md x}$ given the current active set $\support$;
\EndWhile
\While{the sufficient condition in Proposition~\ref{prop:active_sufficient}
  fails and $\abs{\support} < s$}
\State $\support \gets \support \cup \Bigl\{i=\arg\max_{i\notin \support}
  \bigabs{\bigl[\nabla L(\hat{\md x})\bigr]_i}\Bigr\}$;
\State compute $\hat{\md x}$ given the current active set $\support$.
\EndWhile
\end{algorithmic}
\end{algorithm}

\section{Support consistency}\label{sec:consistency}
Having described some efficient approaches to the minimization
problems in (\ref{eq:problems}), we investigate next whether their solutions
provide consistent estimates of the true parameter vector $\md x^*$ in the
classic linear regression framework where
\begin{equation} \label{eq:lossLS}
L(\md x) = \frac{1}{2n}\norm{\md y - \md A \md x}_2^2
         = \frac{1}{2n}\norm{\md A (\md x^* - \md x) + \md w}_2^2.
\end{equation}
Here, the observations $\md y \in \mathbb{R}^n$ are assumed to depend on $\md
x^*$ according to the linear model $\md y = \md A \md x^* + \md w$, where $\md
A\in \mathbb{R}^{n\times p}$ is a (known) feature matrix and $\md w\in
\mathbb{R}^n$ is a noise vector.

As for other Lasso-like estimators, it is expected that (\ref{eq:problems})
promotes sparsity in the parameter vector at the cost of some shrinkage of the
magnitude of their entries \cite{Oymak13}. For this reason, our primary interest is not in bounding the
estimation error but rather in whether the regularized problems in
(\ref{eq:problems}) are capable of reconstructing the signed support of $\md
x^*$, which we will hereafter denote by $\signsupp(\md x^*)$, with
$$
\bigl[\signsupp(\md x)\bigr]_i = \begin{cases}
  +1 &\text{if } x_i > 0 \\
  -1 &\text{if } x_i < 0 \\
   0 &\text{if } x_i = 0.
\end{cases}
$$
Indeed, if the support is recovered correctly, a more precise estimate can be
obtained by solving a restricted optimization problem without regularization
terms.

Our interest will be on the probability of recovering $\signsupp(\md x)$ when $n$
grows large. To that effect, we consider a sequence of feature matrices $(\md A_n)_{n \geq 1}$ 
and observation vectors $(\md y_n)_{n \geq 1}$ of growing dimensions, indexed by $n$. To simplify the
notation, we sometimes omit the subscript $n$ in all these quantities when their dimension is clear from the context. Moreover, we allow $p$ to also depend on $n$,
in the sense that $p=p(n) \rightarrow \infty$ as $n \to \infty$. For each $n$
we consider the estimated parameter vector
\begin{equation}\label{eq:regression}
\hat{\md x}_n = \arg\min_{\md x\in \mathbb{R}^p} \frac{1}{2n}\norm{\md y_n - \md A_n \md
x}_2^2 + \lambda_n \ournorm(\md x)
\end{equation}
where $\lambda_n \to 0$ as $n \to \infty$. The underlying partition $\partition$ in
 the definition of the norm $\ournorm$ is also allowed to depend on $n$, so that both the number of groups and their elements are allowed to depend on $n$ (we choose to obviate this dependence in the notation for simplicity). 

In the following, we will denote by $\md A_{\supportgeneric}$ the submatrix of $\md A_n$ corresponding to
the columns indexed by the set $\supportgeneric \subset [p]$ (this index set is allowed to depend on $n$, even if this is omitted in the notation). Moreover, we will 
make the following assumptions.
\begin{description}
\item[(A1)] The observation $\md y_n$ can be expressed as $\md y_n = \md A_n \md x_n^\ast + \md w_n$ where the entries of $\md w_n$ are independent and identically distributed subgaussian random variables with zero mean and variance $\sigma^2$.

\item[(A2)] Denoting $\support = \supp(\md x_n^*)$, which may generally depend on $n$ (even if we choose to omit this in the notation). Then, the minimum and maximum eigenvalues of
$n^{-1} \md A_{\support}^\tT \md A_{\support}$ are bounded away from zero. In other words, there exists a positive constant $C_{\min}$ such that
$$
\inf_n \lambda_{\min}\biggl(\frac{1}{n} \md A_{\support}^\tT \md A_{\support} \biggr) \geq C_{\min} > 0
$$
where $\lambda_{\min}$ denotes the minimum eigenvalue of an Hermitian matrix. We additionally require that there exists a positive $C_\infty$ constant such that
\begin{equation} \label{eq:norminftyinverse}
    \sup_n \left\| \left( \frac{1}{n} \md A_{\support}^\tT \md A_{\support} \right)^{-1} \right\|_{\ournormrest^\ast \to  \infty} < C_\infty
    < +\infty
\end{equation}
where $\|\cdot\|_{\ournormrest^\ast \to  \infty}$ is the operator norm between the spaces $(\mathbb{R}^{|\support|},\ournormrest^\ast)$ and $(\mathbb{R}^{|\support|},\norm{\cdot}_\infty)$.

\item[(A3)] The absolute value of the non-zero entries of the true parameter vector $\md x_n^\ast$ are 
contained in a compact interval of the positive real axis independent of $n$, which is denoted as $[|x^\ast|_{\min} ,|x^\ast|_{\max} ]$. In other words,
$$
0<|x^\ast|_{\min} \leq \inf_n \min_{j \in \support} |x_{n,j}^\ast| \leq 
\sup_n \max_{j \in \support} |x_{n,j}^\ast| \leq |x^\ast|_{\max} < +\infty
$$
where $x^\ast_{n,j}$ is the $j$th entry of $\md x^\ast_n$. Furthermore, the support of 
$\md x_n^\ast$ does not asymptotically privilege any of the groups, in the sense that 
\begin{equation} \label{eq:balancedblocks}
    \limsup_{n \to \infty} \frac{ \max_{G \in \partition} |G \cap \support|}{\min_{G \in \partition} |G \cap \support|} < +\infty.
\end{equation}

\item[(A4)]
There exists a sequence of positive numbers $\gamma_n \in (0,1) $ such that 
\begin{equation}\label{eq:incoherence}
\min_{G \in \partition} \left( \frac{\norm{\md x_G^\ast}_1}{\ournorm(\md x_n^\ast)} - \left\| \md A_{G \cap \support^\tc}^\tT \md A_{\support} \left(\md A_\support^\tT \md A_\support\right)^{-1} \right\|_{\ournormrest^* \to \infty} \right) > \gamma_n > 0.
\end{equation}
\end{description}

The above assumptions are similar in nature to those that are typically imposed in this type of analysis (see, e.g., \cite{Wainwright09,
Jenatton11b, Obozinski2011}), and the conventional assumptions can be recovered under the 
particular case where the partition $\partition$ consists of only one group. In particular,
the condition in (\ref{eq:norminftyinverse}) generalizes the conventional assumption that 
the $\ell_\infty$-norm of $n (\md A_\support^\tT \md A_\support)^{-1}$ is bounded to the 
general case where the partition $\partition$ is not trivial. More specifically, if we 
use the shorthand notation $\md B = n (\md A_\support^\tT \md A_\support)^{-1}$, we see that
we are interested in matrices $\md B = [\bs\beta_1\quad \bs\beta_2 \quad \cdots \quad
\bs\beta_{|\support|}]^\tT \in \mathbb{R}^{\abs{\support}\times \abs{\support}}$ such that
$$
\norm{\md B}_{\ournormrest^* \to \infty} = \sup_{\substack{\md z\in
\mathbb{R}^{\abs{\support}}\\ \ournormrest^*(\md z) \le 1}}
\norm{\md B \md z}_\infty = \max_{1 \leq i \leq \abs{\support}}\ournormrest(\bs\beta_i)
$$
is bounded in $n$. By Jensen's inequality and the fact that $\sum_i{\abs{x_i}^2}<(\sum_i{\abs{x_i}})^2$, we see that $\abs{\partition}^{-1/2} \norm{\md x_\support}_1 \leq \ournormrest(\md x_\support) \leq \norm{\md x_\support}_1$. In particular, $\norm{\md B}_{\ournormrest^* \to \infty} \leq \norm{\md B}_{\infty\to\infty}$, meaning
that we are generally considering a wider set of matrices as compared to the conventional $\ell_1$ case.

Assumption (A3) implies that the amplitude of
the optimum parameter vector should not completely vanish or grow without bound on any of the groups of the
partition. Furthermore, the support that this parameter vector shares with each of the 
groups of the partition should be asymptotically balanced, in the sense that the support
is not allowed to grow asymptotically faster in any of the groups. In the trivial case where $\partition$ has only one group, this means that the support should not grow at the same 
rate as the potential number of variables $p$.

Finally, assumption (A4) is perhaps the most technical and difficult to interpret. In the
trivial case where $\partition$ has only one group, this condition takes the familiar form
$$
\norm{\md A_{\support^\tc}^\tT \md A_{\support} \left(\md A_\support^\tT \md
A_\support\right)^{-1} }_{\infty\to\infty} < 1- \gamma_n
$$
which has been widely used in e.g.\ \cite{Wainwright09} when $\gamma_n$ is a fixed constant.
In our case, the assumption reflects the fact that we want to allow for the situation where the number of groups of the 
partition $|\partition|$ increases without bound. In this situation, we may have 
$G \in \partition$ for which
$$
\frac{\norm{\md x^\ast_G}_1}{\ournorm(\md x_n^\ast)} \to 0
$$
as $n \to \infty$. Hence, according to (A4) we need to impose that
$$ 
\norm{\md A_{G \cap \support^\tc}^\tT \md A_{\support} \left(\md A_\support^\tT \md A_\support\right)^{-1}}_{\ournormrest^\ast \to \infty} < \frac{\norm{\md x^\ast_G}_1}{\ournorm(\md x_n^\ast)} \to 0
$$ 
so that obviously $\gamma_n \to 0$ as well. The sequence $\gamma_n$ controls the rate at
which the two quantities on either side of the above inequality stay bounded away from one another as $n$ grows large. 

Having presented the main technical assumptions, we now present the sufficiency result for consistency in signed support recovery. 

\begin{theorem}\label{thm:consist_suffic}
Consider the assumptions (A1)--(A4), 
and assume that $\lambda_n$ is chosen so that $\lambda_n \rightarrow 0$ while $\log p / n \to 0$ as $n \to \infty$. Furthermore, assume that there exists an $\eta>0$ such that
\begin{equation}\label{eq:consist_lambda}
\lambda_n^2 \gamma_n^2 n^{1-\eta} \rightarrow \infty.
\end{equation}
Then, for all $n$ sufficiently large,
$$
\Pr (\signsupp(\hat{\md x}_n) = \signsupp(\md x_n^\ast) ) \geq 1 - \exp(-n^\eta).
$$
\end{theorem}
\begin{proof} 
See Appendix \ref{apdx:proof_consistency}.
\end{proof}

The above result imposes some important requirements on the sequence $\lambda_n$ to 
guarantee the asymptotic recovery of the signed support. In particular, the sequence
$\lambda_n$ must converge to zero at a sufficiently slow rate, so that (\ref{eq:consist_lambda}) can be ensured for some $\eta >0$. Let us analyze this
condition under some specific scenarios.

Consider first the case where the number of groups $\abs{\partition}$ is fixed and does
not depend on~$n$. Observe that for each $G \in \partition$ we can write
$$
\frac{\norm{\md x_G^\ast}_1}{\ournormrest(\md x^\ast_n)} \geq \frac{\abs{x^\ast}_{\min}}{\abs{x^\ast}_{\max}}\frac{\abs{G \cap \support}}{\Phi_\support(\partition)}
$$
where we have introduced the quantity
\begin{equation} \label{eq:defPhi}
    \Phi_\support(\partition) = \sqrt{\sum_{G \in \partition} |G \cap \support|^2}.
\end{equation}
According to (\ref{eq:balancedblocks}), we can find a positive constant $K \geq 1$ (independent of $n$) such that
$\max_{G \in \partition} \abs{G \cap \support} \leq K \min_{G \in \partition} \abs{G \cap \support}$ for all large $n$, and consequently 
$$
\frac{\norm{\md x_G\ast}_1}{\ournormrest(\md x^\ast_n)} \geq \frac{\abs{x^\ast}_{\min}}{\abs{x^\ast}_{\max}}\frac{1}{K \sqrt{|\partition|}}>0
$$
for all large $n$. In this situation, (A4) holds provided that we can select a constant $\gamma \in (0, \frac{\abs{x^\ast}_{\min}}{\abs{x^\ast}_{\max}}\frac{1}{K \sqrt{|\partition|}})$ independent of $n$ such that
\begin{equation} \label{eq:sufficientgroups}
\max_{G \in \partition} \left\| \md A_{G \cap \support^\tc}^\tT \md A_{\support} \left(\md A_\support^\tT \md A_\support\right)^{-1} \right\|_{\ournormrest^* \to \infty} < \frac{\abs{x^\ast}_{\min}}{\abs{x^\ast}_{\max}}\frac{1}{K \sqrt{|\partition|}} - \gamma.
\end{equation}
Assuming that the measurement matrix $\md A$ is properly designed so that the above 
incoherence condition holds, we can build $\lambda_n$ as any sequence such that $\lambda_n \to 0$ while $\lambda_n^2 n^{1-\eta} \to \infty$, which essentially are the conventional conditions that need to be imposed in the trivial case where $\partition$ has only one group. In particular, $\lambda_n = O(n^{-k})$ for $0<k<(1-\eta)/2$ will work. 

The case where the number of groups $\abs{\partition}$ is allowed to increase with $n$ is 
far more interesting. Assume that the measurement matrix $\md A$ is designed so that 
the incoherence condition in (\ref{eq:sufficientgroups}) holds with $\gamma$ replaced by
a positive sequence decreasing to zero such that $\gamma_n < \frac{\abs{x^\ast}_{\min}}{\abs{x^\ast}_{\max}}\frac{1}{K \sqrt{|\partition|}}$ for each $n$. In particular, we see that $\gamma_n$ should go to zero at a rate at least faster than
$O(\abs{\partition}^{-1/2})$. Now, assuming that we choose $\lambda_n = O(n^{-k})$ and $\gamma_n = O(n^{-r})$ for
some $k,r>0$, in order to verify (\ref{eq:consist_lambda}) we must have 
that $k+r < (1-\eta)/2$. In particular, one must have $r < 1/2$ and $\gamma_n = o(n^{-1/2})$, so that necessarily one must have $\abs{\partition} = o(n)$ in order
to guarantee asymptotic signed support recovery.

\section{Numerical results}\label{sec:numresults}

We consider here a numerical evaluation of the structured linear regression problem in which the columns of the measurement matrix $\md A$ consisted of $p = 200$ columns independently generated according to a uniform distribution in the unit sphere. The number of measurements $n$ is either $n=110$ or $n=180$ and for each choice we ran $200$ experiments changing the corresponding measurement matrix $\md A$. The parameter vector $\md x$ consisted of two strings of all ones with length equal to $10$ starting at positions $4$ and $173$. The observation was generated according to the equation $\md y = \md A \md x + \md w$ where $\md w$ was zero-mean Gaussian distributed with zero mean and covariance $\sigma^2 \md I_n$, with $\sigma^2=1$ or $\sigma^2=0.01$.

Different recovery algorithms were tested, all of them penalized versions of the loss function in (\ref{eq:lossLS}) with  regularizer parameter ranging from $5e-3$ to $2.5$. Following the notation in (\ref{eq:problems}), the regularizer parameter is denoted as $\lambda$ when the penalty is a norm and $\mu$ when the penalty is the square of a norm. We considered also the following algorithms.

\begin{description}
    \item[Exclusive (Proximal):] The exclusive group Lasso regularizer in (\ref{eq:norm_def}), solved with the FISTA-inspired proximal algorithm of Section~\ref{sec:proximal}.
    \item[Exclusive (Active Set):] The exclusive group Lasso regularizer in (\ref{eq:norm_def}), implemented according to the proposed active set algorithm in Section~\ref{sec:active_set}, without enforcing the string structure of the parameters in the evolution path.
    \item[Exclusive (Active Strings):] The exclusive group Lasso regularizer in
    (\ref{eq:norm_def}), also implemented with the proposed active set algorithm
    in Section~\ref{sec:active_set} but now enforcing the string structure in
    the evolution path (i.e. only connecting the orange boxes in \fig~\ref{fig:evolution}).
    \item[Classic:] The conventional Lasso regularizer ($\ell_1$ norm) implemented through the fast iterative shrinkage/thresholding algorithm (FISTA) \cite{Beck09}.
    \item[Latent:] The latent Group Lasso penalty in (\ref{eq:latentgrupLASSO}) solved according to the FISTA approach (see \cite{Villa2014}).
    \item[Structured (Active Set):] The group Lasso regularizer in (\ref{eq:svspenalty}) implemented according to the corresponding active set algorithm in \cite{Jenatton11}.
\end{description}

The groups are taken differently depending on the algorithm. In the exclusive group Lasso approaches, the groups are selected as the $10$ subsets of $[p]$ such that, for $i=0,\ldots,9$, the $i$th subset contains all possible indices of $[p]$ congruent with $i$ modulo $10$. In the latent group Lasso approach, the groups are taken as the $p$ possible subsets of $[p]$ containing exactly $10$ consecutive indices (modulo $p$). In the classic Lasso approach, obviously, no underlying structure is specified. For the structured Lasso, the groups are those described in \cite{Jenatton11}.

Recalling the activation pattern of the parameter vector (see above), one
readily sees that this consists of two active entries per group in the case of
the exclusive group Lasso, and two active groups in the case of the latent group
Lasso, while there is no real match for the structured Lasso.

\begin{figure}
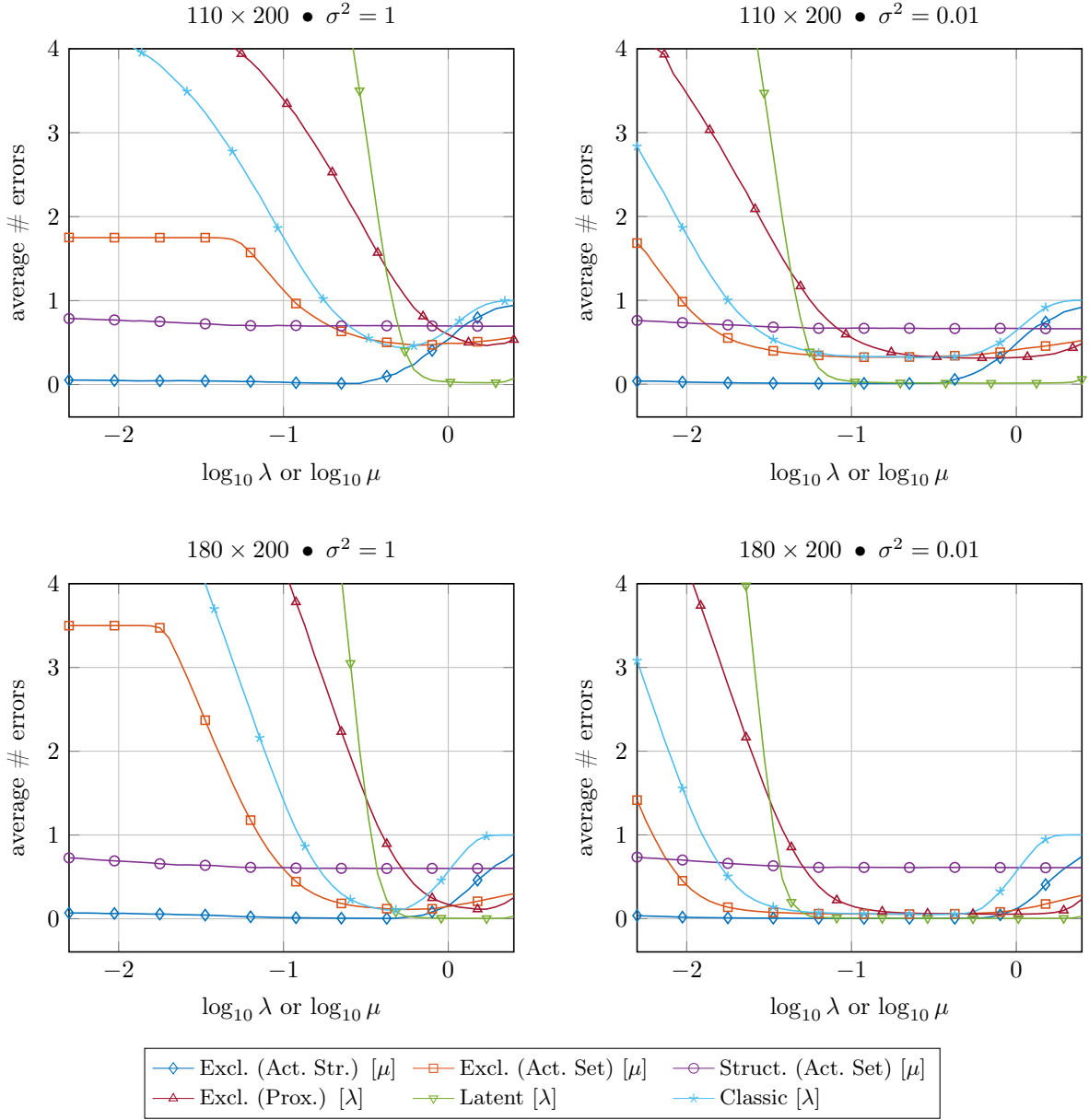

\centering
\inputtikz{figs/single/errors_db20}
\caption{Average number of errors in the recovery of the signed support for
different numbers of observations (top--bottom) and different noise variances
(left--right).}\label{fig:errors}
\end{figure}

\fig~\ref{fig:errors} shows the average number of errors in the recovery of the
signed support for the six strategies above. We can observe that, when the
regularizer parameter (either $\lambda$ or~$\mu$) is properly set, both
the algorithms presented in this paper---``Exclusive (Proximal)'' and ``Exclusive
(Active Set)''---perform at least as well as the ``Classic'' Lasso regularizer.
Note that no major gains were expected since the structure promoted by the
exclusive group sparsity norm (that is, few uncorrelated active entries inside
each group) is rather weak and not much more constraining than plain sparsity.

Nevertheless, the exclusive group sparsity is quite flexible and more rigid
structures can be built upon it by including extra information. For instance,
looking at the activation pattern of the parameter vector, we see that not only
does it imply two active entries per group, but it also requires that entries
from different groups are activated consecutively, forming a string. The
``Active String'' algorithm (a minor modification of the ``Active Set'') has
been designed to enforce this structure and, even though empirical, it provides
the same performance as the ``Latent'' group Lasso, the natural choice for the
activation pattern at hand.

As a final comment, the poor performance of the ``Structured'' Lasso algorithm
was also expected since, as mentioned before, the promoted sparsity pattern (a
single, continuous string) is not a match with the simulated activation pattern.

\section{Conclusions} \label{sec:conclusions}
The paper provides a thorough characterization of the exclusive group sparsity
norm and derives two efficient methods for solving optimization problems that
include such norm as a regularizer. The effectiveness of the norm in promoting
the desired sparsity is proven both theoretically, by studying asymptotic
estimation consistency, and through simulations. Even though the exclusive group
sparsity by itself does not offer a strong structure to exploit for support
recovery, its flexibility allows tailoring the proposed
optimization algorithms (specifically, the active set one) to fit more
complex sparsity patterns and achieve performances that are as good as those of
rigid, dedicated regularizers.

\acks{This work has been supported by the Spanish Government under grants
TEC2014-59255-C3-1-R and RTI2018-099722-B-I00.}

\appendix
\section{General results}\label{apdx:proof_general}
This appendix provides a set of theoretical tools which the proofs of
Appendices~\ref{apdx:proof_active_set} and~\ref{apdx:proof_consistency} are
built upon.

\subsection{Notation and preliminary facts}\label{apdx:preliminary}
Aiming at a self-contained paper, we review here some basic properties of the
problems of the form
$$
\minimize_{\md x\in \mathbb{R}^p} L(\md x) + \frac{\mu}{2} \Phi^2(\md x)
$$
where $\Phi(\cdot)$ may be any norm in $\mathbb{R}^p$ and not
necessarily the one defined in (\ref{eq:norm_def}). By doing so, we will also
introduce some notation that will prove useful in the following. It is worth
mentioning that the results of this section are well known and a proof can be
found in, e.g.,~\cite{Jenatton11b}.

Being interested in solutions with sparse support, we often deal with restricted
versions of the loss function and of the norm. Namely, for a generic index set
$\supportgeneric \subseteq [p]$, the restricted loss function $L_\supportgeneric(\md
x_\supportgeneric): \mathbb{R}^{\abs{\supportgeneric}} \to \mathbb{R}_+$
maps $\md x_\supportgeneric$ to $L_\supportgeneric(\md x_\supportgeneric) =
L({\md x_{\{\supportgeneric\}}})$, where we recall that ${\md x_{\{\supportgeneric\}}}\in\mathbb{R}^p$ is the ``zero-padded''
extension of $\md x_\supportgeneric$, that is $(\md x_{\{\supportgeneric\}})_\supportgeneric =
\md x_\supportgeneric$ and $(\md x_{\{\supportgeneric\}})_{\supportgeneric^\tc} = \md 0$. It is
straightforward to see that $L_\supportgeneric(\cdot)$ preserves the convexity
and differentiability of $L(\cdot)$.

The restricted norm $\Phi_\supportgeneric(\md x_\supportgeneric)$, which is a
proper norm, is defined accordingly. It is worth mentioning that, for the
specific case of the exclusive group sparsity norm defined in (\ref{eq:norm_def}), the
restricted norm takes the form
$$
\ournormrestgeneric(\md x_\supportgeneric) = \sqrt{\sum_{G\in\partition_\supportgeneric}
\norm{\md x_{G\cap\supportgeneric}}_1^2}
$$
where we recall that $\partition_\supportgeneric = \{G \in \partition : G\cap \supportgeneric
\ne \emptyset\}$ is the subset of $\partition$ with only groups that contain at
least one index of $\supportgeneric$.

Finally, for a dual characterization of the optimization problem, we also need
to introduce two new functions $\mathbb{R}^{\abs{\supportgeneric}}\to
\mathbb{R}_+$, namely
\begin{align*}
L_\supportgeneric^*(\md u_\supportgeneric) &= \sup_{\md x_\supportgeneric \in
  \mathbb{R}^{\abs{\supportgeneric}}} \md u_\supportgeneric^\tT \md
  x_\supportgeneric - L_\supportgeneric(\md x_\supportgeneric) \\
\Phi_\supportgeneric^*(\md u_\supportgeneric) &= \sup_{\md x_\supportgeneric :
  \Phi_\supportgeneric(\md x_\supportgeneric) \le 1} \md u_\supportgeneric^\tT
  \md x_\supportgeneric
\end{align*}
that is the Fenchel conjugate of $L_\supportgeneric(\md x_\supportgeneric)$ and
the dual norm of $\Phi_\supportgeneric(\md x_\supportgeneric)$, respectively.
Note that, even though $L_\supportgeneric(\cdot)$ and
$\Phi_\supportgeneric(\cdot)$ extend to $L(\cdot)$ and $\Phi(\cdot)$ in the
sense explained above, the same property does not hold, in general, for the
Fenchel conjugates, $L_\supportgeneric^*(\cdot)$ and $L^*(\cdot)$, or for the
dual norms, $\Phi_\supportgeneric^*(\cdot)$ and $\Phi^*(\cdot)$.

We are now in possession of all the tools needed to analyze the restricted
minimization problem and its Fenchel dual, which read
\begin{gather*}
\minimize_{\md x_\supportgeneric\in \mathbb{R}^{\abs{\supportgeneric}}}
  L_\supportgeneric(\md x_\supportgeneric) + \frac{\mu}{2}
  \Phi_\supportgeneric^2(\md x_\supportgeneric) \\
\minimize_{\md u_\supportgeneric\in \mathbb{R}^{\abs{\supportgeneric}}}
  L_\supportgeneric^*(\md u_\supportgeneric) + \frac{1}{2\mu}
  [\Phi_\supportgeneric^*(\md u_\supportgeneric)]^2
\end{gather*}
where we used the fact that the Fenchel conjugate of $\frac{\mu}{2}
\Phi_\supportgeneric^2(\cdot)$ is
$\frac{1}{2\mu}[\Phi_\supportgeneric^*(\cdot)]^2$ \cite[Example 3.27]{Boyd04}.
From here, one can
straightforwardly show (e.g., \cite[Proposition~5.3.8]{Bertsekas_Convex}) that
strong duality holds and that the primal--dual
optimal points $(\hat{\md x_\supportgeneric}, \hat{\md u_\supportgeneric})$
are the only ones for which
\begin{equation}\label{eq:casesFY}
\begin{cases}
\hat{\md u}_{\supportgeneric} = \nabla L_{\supportgeneric}(\hat{\md x}_{\supportgeneric})\\
\hat{\md u}_{\supportgeneric} \in - \partial\Bigl[\frac{\mu}{2}
    \Phi_\supportgeneric^2\Bigr](\hat{\md x}_{\supportgeneric})
  = - \mu \Phi_\supportgeneric(\hat{\md x}_{\supportgeneric})
    \partial\Phi_\supportgeneric(\hat{\md x}_{\supportgeneric}).
\end{cases}
\end{equation}
To understand the above point, we can simply realize that the duality gap of this problem
is given by 
\begin{equation} \label{eq:dualitygap2terms}
\left( L_{\supportgeneric}(\md x_{\supportgeneric}) + L^\ast_{\supportgeneric}(\md u_{\supportgeneric}) - \md x_{\supportgeneric}^\tT \md u_{\supportgeneric} \right)
+ \left( \frac{\mu}{2} \Phi_\supportgeneric^2 \left({\md x}_{\supportgeneric}\right) +
\frac{1}{2\mu} \left[\Phi^\ast_\supportgeneric \left({\md u}_{\supportgeneric}\right)\right]^2 + \md x_{\supportgeneric}^\tT \md u_{\supportgeneric}  \right)
\end{equation}
where the two terms inside parentheses are non-negative due to the Fenchel-Young 
inequality. Strong duality holds when we have equality in the two terms, which
happens when the dual variable of each term ($\md u_\supportgeneric$ and $-\md 
u_\supportgeneric$ respectively) belongs to the subdifferential of the corresponding
primal functions. In particular, the fist and second terms in parentheses respectively
vanish if and only if the first and second conditions in (\ref{eq:casesFY}) hold true. 

Next, we reason that the second condition in (\ref{eq:casesFY}) is equivalent to
\begin{equation} \label{eq:3ident}
    -\hat{\md x}_{\supportgeneric}^\tT \hat{\md u}_{\supportgeneric} = \mu \Phi_\supportgeneric^2 \left(\hat{\md x}_{\supportgeneric}\right) =
\frac{1}{\mu} \left[\Phi^\ast_\supportgeneric \left(\hat{\md u}_{\supportgeneric}\right)\right]^2. 
\end{equation}
On the one hand, as a direct consequence of the Fenchel-Young inequality, this condition 
holds if and only if
\begin{equation} \label{eq:FYequality}
-\hat{\md x}_{\supportgeneric}^\tT \hat{\md u}_{\supportgeneric} = \frac{\mu}{2} \Phi_\supportgeneric^2 \left(\hat{\md x}_{\supportgeneric}\right) +
\frac{1}{2\mu} \left[\Phi^\ast_\supportgeneric \left(\hat{\md u}_{\supportgeneric}\right)\right]^2 
\end{equation}
On the other hand, we recall from (\ref{eq:subdiff}) that
$$
\partial\Phi_\supportgeneric(\hat{\md x}_{\supportgeneric}) = \{\md
z_{\supportgeneric} \in \mathbb{R}^{\abs{\supportgeneric}}:
\md z_{\supportgeneric}^\tT \hat{\md x}_{\supportgeneric} =
\Phi_\supportgeneric(\hat{\md x}_{\supportgeneric}),\quad \Phi_\supportgeneric^*(\md
z_\supportgeneric) \le 1\}.
$$
Therefore, we have $\hat{\md u}_\supportgeneric \in -\mu \Phi_\supportgeneric(\hat{\md x}_{\supportgeneric}) \partial\Phi_\supportgeneric(\hat{\md x}_{\supportgeneric})$ if and
only if $\hat{ \md u}_\supportgeneric$ is such that $\hat{ \md u}_\supportgeneric ^\tT 
\hat{ \md x}_\supportgeneric = - \mu \Phi^2_\supportgeneric(\hat{\md x}_\supportgeneric)$ 
and $ \Phi_\supportgeneric^*(\hat{\md u}_\supportgeneric) \le \mu 
\Phi_\supportgeneric(\hat{\md x}_\supportgeneric)$. 

In summary, we see that the second condition in (\ref{eq:casesFY}) implies that 
$$
\frac{\mu}{2} \Phi_\supportgeneric^2(\hat{\md x}_\supportgeneric)
  + \frac{1}{2\mu}[\Phi_\supportgeneric^*(\hat{\md u}_\supportgeneric)]^2
= -\hat{\md x}_\supportgeneric^\tT \hat{\md u}_\supportgeneric
= \mu \Phi_\supportgeneric^2(\hat{\md x}_{\supportgeneric}).
$$
which directly leads to (\ref{eq:3ident}). Conversely, to see that the identity in 
(\ref{eq:3ident}) implies the second condition in (\ref{eq:casesFY}) we only need to
see that $ \Phi_\supportgeneric^*(\hat{\md u}_\supportgeneric) \le \mu 
\Phi_\supportgeneric(\hat{\md x}_\supportgeneric)$. But this follows directly from 
the definition of dual norm, which implies that $ -\hat{\md x}_{\supportgeneric}^\tT \hat{\md u}_{\supportgeneric} \leq  \Phi_\supportgeneric(\hat{\md x}_\supportgeneric) \Phi_\supportgeneric^*(\hat{\md u}_\supportgeneric)$, so that using (\ref{eq:FYequality})
we obtain 
$$
\frac{1}{\mu} \left[\Phi^\ast_\supportgeneric \left(\hat{\md u}_{\supportgeneric}\right)\right]^2 = -\hat{\md x}_{\supportgeneric}^\tT \hat{\md u}_{\supportgeneric} \leq  \Phi_\supportgeneric(\hat{\md x}_\supportgeneric) \Phi_\supportgeneric^*(\hat{\md u}_\supportgeneric)
$$
as we wanted to show. Consequently, the primal--dual optimal points $(\hat{\md x}_\supportgeneric,
\hat{\md u}_\supportgeneric)$ must satisfy
\begin{equation}\label{eq:primal-dual}
\begin{cases}
\hat{\md u}_{\supportgeneric} = \nabla L_{\supportgeneric}(\hat{\md x}_{\supportgeneric})\\
\hat{\md x}_{\supportgeneric}^\tT \hat{\md u}_{\supportgeneric} = -\frac{1}{\mu}
  [\Phi_{\supportgeneric}^*(\hat{\md u}_{\supportgeneric})]^2 =
  -\mu\Phi_{\supportgeneric}^2(\hat{\md x}_{\supportgeneric}).
\end{cases}
\end{equation}
In more general terms, assume that we have a pair of values 
$(\md x_\supportgeneric, \md u_\supportgeneric)$ that are not optimal points, but 
still chosen so that $\md u_\supportgeneric = \nabla L_\supportgeneric(\md x_\supportgeneric)$.
According to the above reasoning, the first term in (\ref{eq:dualitygap2terms}) is zero,
so that the duality gap of the reduced problem takes the form 
$$
\frac{\mu}{2}\Phi_\supportgeneric^2(\md x_\supportgeneric) +
   \frac{1}{2\mu}[\Phi_\supportgeneric^*(\md u_\supportgeneric)]^2
   + \md x_\supportgeneric^\tT \md u_\supportgeneric.
$$

To conclude this section, we point out that these results hold for any index set
$\supportgeneric$ and, in particular, also for $\supportgeneric = [p]$, which corresponds
to the original (complete) problem. Then, we
can characterize how well a primal--dual optimal point $(\hat{\md x}_\supportgeneric,
\hat{\md u}_\supportgeneric)$ of the restricted problem behaves as a solution to the
original problem. We recall here that $\hat{\md x}_{\{\supportgeneric\}}$ and 
$\hat{\md u}_{\{\supportgeneric\}}$ contain the primal/dual variable solutions of the
restricted problem in the positions indexed by $\supportgeneric$ and zeros elsewhere. 
Now, since $\md u_\supportgeneric = \nabla L_\supportgeneric(\md x_\supportgeneric)$ we
must have $\md u_{\{\supportgeneric\}}
= \nabla L(\md x_{\{\supportgeneric\}})$.
Therefore, the solution to the restricted problem will achieve a duality gap:
\begin{align}
    \frac{\mu}{2}\Phi^2(\hat{\md x}_{\{\supportgeneric\}}) &+
  \frac{1}{2\mu}[\Phi^*(\hat{\md u}_{\{\supportgeneric\}})]^2
  + \hat{\md x}_{\{\supportgeneric\}}^\tT \hat{\md u}_{\{\supportgeneric\}} 
  = \frac{\mu}{2}\Phi^2(\hat{\md x}_{\{\supportgeneric\}}) +
  \frac{1}{2\mu}[\Phi^*(\hat{\md u}_{\{\supportgeneric\}})]^2
  + \hat{\md x}_{\supportgeneric}^\tT \hat{\md u}_{\supportgeneric} \nonumber \\
  &= \frac{\mu}{2}\Phi^2(\hat{\md x}_{\{\supportgeneric\}}) +
  \frac{1}{2\mu}[\Phi^*(\hat{\md u}_{\{\supportgeneric\}})]^2
  - \frac{\mu}{2} \Phi_\supportgeneric^2(\hat{\md x}_\supportgeneric)
  - \frac{1}{2\mu} [\Phi_\supportgeneric^*(\hat{\md u}_\supportgeneric)]^2 \label{eq:dualitygaprest} \\
   &= \frac{1}{2\mu}[\Phi^*(\hat{\md u}_{\{\supportgeneric\}})]^2
   - \frac{1}{2\mu} [\Phi_\supportgeneric^*(\hat{\md u}_\supportgeneric)]^2 \label{eq:gap}
\end{align}
where in (\ref{eq:dualitygaprest}) we have used the fact that the solution to
the restricted problem achieves a zero duality gap and where in (\ref{eq:gap}) we have 
used the definition of restricted norm.

\subsection{Solution characterization}
The following lemma further characterizes the solutions to~(\ref{eq:problem_norm2}) in
terms of their support.

\begin{lemma}\label{lemma:active_basic}
Vector $\hat{\md x}$, with support $\support = \supp\{\hat{\md x}\}$, is a solution to
$$
\minimize_{\md x\in \mathbb{R}^p} L(\md x) + \frac{\mu}{2} \ournorm^2(\md x)
$$
if and only if
\begin{equation}\label{eq:basic_first}
\begin{cases}
[\nabla L(\hat{\md x})]_{\support} + \mu \hat{\md r}_{\support} = \md 0\\
\max_{G\in\partition_{\support}} \frac{\bignorm{[\nabla L(\hat{\md x})]_{G\cap
    \support^\toff}}_{\infty}}{\lVert \hat{\md x}_{G\cap\support} \rVert_1}
  \le \mu \\
[\nabla L(\hat{\md x})]_{\support^\tinact} = \md 0
\end{cases}
\end{equation}
where $\hat{\md r}\in\mathbb{R}^p$ is a vector with entries
\begin{equation}\label{eq:vector_r}
\hat{r}_i = \begin{cases}
\lVert \hat{\md x}_{G_i\cap\support} \rVert_1 \sign(\hat{x}_i)
  &\text{if } i \in \support\\
0 &\text{if } i \in \support^\tc
\end{cases}
\end{equation}
with $G_i$ the (unique) group containing index $i$, and where $\support^\toff = \{i\in
\support^\tc\cap G, \forall G \in \partition_{\support}\}$ is the set of
inactive entries belonging to active groups. Finally, we denote by
$\support^\tinact = \support^\tc\setminus\support^\toff$ the set of indices
belonging to inactive groups.

Moreover, solution $\hat{\md x}$ also satisfies the following inequalities:
$$
\begin{cases}
\ournormrest^*\bigl([\nabla L(\hat{\md x})]_{\support}\bigr) \le \mu
 \ournorm(\hat{\md x})\\
\ournormrestoff^*\bigl([\nabla L(\hat{\md x})]_{\support^\toff}\bigr) \le
 \mu \ournorm(\hat{\md x})\\
\ournorm^*(\nabla L(\hat{\md x})) \le \mu \ournorm(\hat{\md x}).
\end{cases}
$$
\end{lemma}

\begin{proof}
For the purpose of this proof, let $F(\md x) = L(\md x) +
\frac{\mu}{2}\ournorm^2(\md x)$. Then, by convexity, vector $\hat{\md x}$ is a
solution to (\ref{eq:problem_norm2}) if and only if
$$
\lim_{\epsilon\to 0^+} \frac{F(\hat{\md x} + \epsilon \md w) - F(\hat{\md x})}{\epsilon} \ge
0
$$
for all directions $\md w\in \mathbb{R}^p$. In order to see the implications of
the above inequality, we first need to show that it can be equivalently
rewritten as
\begin{equation}\label{eq:limit_equiv}
\md w^\tT \nabla L(\hat{\md x}) + \mu \hat{\md r}_\support^\tT \md w_\support +
\mu \sum_{G\in\partition_{\support}} \norm{\hat{\md x}_{G\cap\support}}_1
\norm{\md w_{G\cap\support^\tc}}_1 \ge 0.
\end{equation}
Indeed, recalling the definition of $\ournorm(\cdot)$, one has
\begin{align*}
\ournorm^2(\hat{\md x}+\epsilon \md w) &= \sum_{G\in\partition}\Biggl(\sum_{i\in G} \abs{\hat{x}_i
+ \epsilon w_i}\Biggr)^2\\
&= \sum_{G\in\partition_{\support}}\Biggl(\sum_{i\in G\cap\support} \abs{\hat{x}_i
+ \epsilon w_i} + \epsilon\sum_{i\in G\cap\support^\tc}\abs{w_i}\Biggr)^2
+ \epsilon^2\sum_{G\in\partition_{\support}^\tc}\Biggl(\sum_{i\in
G}\abs{w_i}\Biggr)^2.
\end{align*}
Now, let us define $\md s=\sign(\hat{\md x})$ element-wise. Moreover, since
$\epsilon$ is arbitrarily small, we can assume that $\hat{x}_i + \epsilon w_i$
keeps the same sign as $\hat{x}_i$. It follows that
\begin{align*}
\ournorm^2(\hat{\md x}+\epsilon \md w)
&= \sum_{G\in\partition_{\support}}\Bigl(\norm{\hat{\md x}_{G\cap\support}}_1
+ \epsilon \md s_{G\cap\support}^\tT \md w_{G\cap\support} +
\epsilon\norm{\md w_{G\cap\support^\tc}}_1\Bigr)^2 +
\epsilon^2\sum_{G\in\partition_{\support}^\tc}\norm{\md w_G}_1^2\\
&= \ournormrest^2(\hat{\md x}_{\support}) + 2\epsilon
\sum_{G\in\partition_{\support}}
\norm{\hat{\md x}_{G\cap\support}}_1 \Bigl(\md s_{G\cap\support}^\tT \md w_{G\cap\support} +
\norm{\md w_{G\cap\support^\tc}}_1\Bigr) + o(\epsilon) \\
&= \ournormrest^2(\hat{\md x}_{\support})
+ 2\epsilon \hat{\md r}_\support^\tT \md w_\support + 2\epsilon
\sum_{G\in\partition_{\support}} \norm{\hat{\md x}_{G\cap\support}}_1
\norm{\md w_{G\cap\support^\tc}}_1 + o(\epsilon)
\end{align*}
where we have used the definition of $\hat{\md r}$ given in~(\ref{eq:vector_r}).
By plugging the last line into the definition of $F(\hat{\md x} + \epsilon \md
w)$, one readily obtains~(\ref{eq:limit_equiv}).

Now, (\ref{eq:basic_first}) is a straightforward consequence of the fact that
(\ref{eq:limit_equiv}) must hold for all $\md w$ and, specifically,
$$
\inf_{\md w:\supp\{\md w\}\subseteq\mathcal{I}}\Bigl\{\md w^\tT \nabla L(\hat{\md x})
  + \mu \hat{\md r}_\support^\tT \md w_\support +
  \mu \sum_{G\in\partition_{\support}} \norm{\hat{\md x}_{G\cap\support}}_1
  \norm{\md w_{G\cap\support^\tc}}_1\Bigr\} \ge 0
$$
for $\mathcal{I}\in\{\support, \support^\toff, \support^\tinact\}$.

To prove the second part of the lemma, let us start by noting that
\begin{align}
\md w_{\support}^\tT \hat{\md r}_{\support}
&= \sum_{G\in\partition_{\support}} \lVert \hat{\md x}_{G\cap\support}\rVert_1
  \sum_{i\in G\cap\support} w_i \sign(\hat{x}_i) \nonumber\\
&\le \sum_{G\in\partition_{\support}} \lVert \hat{\md x}_{G\cap\support}\rVert_1
  \sum_{i\in G\cap\support} \lvert w_i \rvert \nonumber\\
&= \sum_{G\in\partition_{\support}} \lVert \hat{\md x}_{G\cap\support}\rVert_1
  \lVert \md w_{G\cap\support} \rVert_1 \label{eq:lemma1a}\\
&\overset{(\mathrm{a})}{\le} \sqrt{\sum_{G\in\partition_{\support}} \lVert \hat{\md x}_{G\cap\support}\rVert_1^2}
  \sqrt{\sum_{G\in\partition_{\support}} \lVert \md w_{G\cap\support} \rVert_1^2} \nonumber\\
&= \ournormrest(\hat{\md x}_{\support}) \ournormrest(\md w_{\support}) \nonumber
\end{align}
where $(\mathrm{a})$ is an application of the Cauchy--Schwartz inequality, which
also applies to the last term of (\ref{eq:limit_equiv}) yielding
$$
\sum_{G\in\partition_{\support}} \lVert \hat{\md x}_{G\cap\support} \rVert_1
\lVert \md w_{G\cap\support^\tc} \rVert_1
\le \ournormrest(\hat{\md x}_{\support}) \ournormrestoff(\md
w_{\support^\toff}).
$$
As a result, the left-hand side of (\ref{eq:limit_equiv}) can be upper-bounded
by
\begin{align*}
\md w^\tT\nabla L(\hat{\md x}) &+ \mu \ournormrest(\hat{\md x}_{\support})
\Bigl(\ournormrest(\md w_{\support}) +
\ournormrestoff(\md w_{\support^\toff})\Bigr)\\
&\ge \md w^\tT\nabla L(\hat{\md x}) +
  \mu \hat{\md r}_{\support}^\tT \md w_{\support} + \mu
  \sum_{G\in\partition_{\support}} \lVert \hat{\md x}_{G\cap\support} \rVert_1
  \lVert \md w_{G\cap\support^\tc} \rVert_1 \ge 0.
\end{align*}
Since, once again, the above inequality must hold for all $\md w \in
\mathbb{R}^p$ and, by taking the infimum over all $\md w$ with support in $\support$
and $\support^\toff$, we obtain
$$
\begin{cases}
\ournormrest^*\bigl([\nabla L(\hat{\md x})]_{\support}\bigr) \le \mu
 \ournorm(\hat{\md x})\\
\ournormrestoff^*\bigl([\nabla L(\hat{\md x})]_{\support^\toff}\bigr) \le
 \mu \ournorm(\hat{\md x})
\end{cases}
$$
respectively. The same trick can be applied to show
$$
\ournorm^*(\nabla L(\hat{\md x})) \le \mu \ournorm(\hat{\md x}).
$$
Indeed, one has
\begin{align*}
\ournorm(\md w) &= \sqrt{\sum_{G\in\partition} \norm{\md w_G}_1^2}\\
&\ge  \sqrt{\sum_{G\in\partition_{\support}} \norm{\md w_G}_1^2}\\
&= \frac{1}{\ournormrest(\hat{\md x}_{\support})}
  \sqrt{\sum_{G\in\partition_{\support}} \lVert \hat{\md x}_{G\cap\support}\rVert_1^2}
  \sqrt{\sum_{G\in\partition_{\support}} \lVert \md w_{G} \rVert_1^2} \\
&\ge \frac{1}{\ournormrest(\hat{\md x}_{\support})}
  \sum_{G\in\partition_{\support}} \lVert \hat{\md x}_{G\cap\support}\rVert_1
  \lVert \md w_{G} \rVert_1 \\
&= \frac{1}{\ournormrest(\hat{\md x}_{\support})}
  \sum_{G\in\partition_{\support}} \lVert \hat{\md x}_{G\cap\support}\rVert_1
  \Bigl(\norm{\md w_{G\cap\support}}_1 + \norm{\md w_{G\cap\support^\tc}}_1\Bigr) \\
&\ge \frac{\hat{\md r}_\support^\tT \md w_\support}{\ournormrest(\hat{\md x}_{\support})}
  + \frac{1}{\ournormrest(\hat{\md x}_{\support})}
  \sum_{G\in\partition_{\support}} \lVert \hat{\md x}_{G\cap\support}\rVert_1
  \norm{\md w_{G\cap\support^\tc}}_1
\end{align*}
which implies
$$
\md w^\tT\nabla L(\hat{\md x}) + \mu \ournormrest(\hat{\md x}_{\support})
\ournorm(\md w)
\ge \md w^\tT\nabla L(\hat{\md x}) +
  \mu \hat{\md r}_{\support}^\tT \md w_{\support} + \mu
  \sum_{G\in\partition_{\support}} \lVert \hat{\md x}_{G\cap\support} \rVert_1
  \lVert \md w_{G\cap\support^\tc} \rVert_1 \ge 0.
$$
Finally, we only need to take the infimum over all possible $\md w$ and recall
that $\ournormrest(\hat{\md x}_{\support}) = \ournorm(\hat{\md x})$ since
$\supp\{\hat{\md x}\} = \support$.
\end{proof}

For the proof of Proposition~\ref{prop:active_sufficient}, we will also make use
of the following simple result.
\begin{lemma}\label{lemma:dual_upper_bound}
For any vector $\md u\in\mathbb{R}^p$ and any index subset $\support \in [p]$,
the dual norm $\ournorm^*(\md u)$ can be upper-bounded by
$$
\ournorm^*(\md u) \le \max\{\ournormrest^*(\md u_{\support}),
\ournormrestoff^*(\md u_{\support^\toff}),
\ournormrestinact^*(\md u_{\support^\tinact})\}.
$$
\end{lemma}
\begin{proof}
By definition of dual norm, we have
\begin{align*}
\ournorm^*(\md u) &= \max_{\md v\in\mathbb{R}^p} \md u^\tT \md v\quad\text{s.to }
    \ournorm(\md v)\le 1\\
  &= \max_{\md v\in\mathbb{R}^p} \md u^\tT \md v\quad\text{s.to }
    \sum_{G\in\partition}\norm{\md v_G}_1^2 \le 1\\
  &= \max_{\md v\in\mathbb{R}^p} \md u^\tT \md v\quad\text{s.to }
    \sum_{G\in\partition_\support}\norm{\md v_G}_1^2 +
    \sum_{G\in\partition_{\support^\tinact}}\norm{\md v_G}_1^2 \le 1\\
  &= \max_{\md v\in\mathbb{R}^p} \md u^\tT \md v\quad\text{s.to }
    \sum_{G\in\partition_\support}\Bigl(\norm{\md v_{G\cap\support}}_1
      + \norm{\md v_{G\cap\support^\toff}}_1\Bigr)^2 +
    \sum_{G\in\partition_{\support^\tinact}}\norm{\md v_G}_1^2 \le 1\\
  &\le \max_{\md v\in\mathbb{R}^p} \md u^\tT \md v\quad\text{s.to }
    \sum_{G\in\partition_\support}\norm{\md v_{G\cap\support}}_1^2
    + \sum_{G\in\partition_\support} \norm{\md v_{G\cap\support^\toff}}_1^2 +
    \sum_{G\in\partition_{\support^\tinact}}\norm{\md v_G}_1^2 \le 1 \\
  &= \max_{\md v\in\mathbb{R}^p,a,b,c} \md u^\tT \md v \\
    &\phantom{= mt}\text{s.to }
    \sum_{G\in\partition_\support}\norm{\md v_{G\cap\support}}_1^2 \le a,\quad
    \sum_{G\in\partition_\support} \norm{\md v_{G\cap\support^\toff}}_1^2 \le b,
    \quad \sum_{G\in\partition_{\support^\tinact}}\norm{\md v_G}_1^2 \le c,\\
    &\phantom{= mt\text{s.to }} a+b+c \le 1.
\end{align*}
Since sets $\support$, $\support^\toff$ and $\support^\tinact$ form a partition
of $[p]$, the result is straightforward (see also \cite[Lemma~15]{Jenatton11b}).
\end{proof}

\section{Active set: Proofs}\label{apdx:proof_active_set}
This section provides a proof of Propositions~\ref{prop:active_necessary}
and~\ref{prop:active_sufficient}, the two building blocks of the active set
algorithm described in Section~\ref{sec:active_set}.

\subsection{Proof of Proposition~\ref{prop:active_necessary}}
Suppose vector $\hat{\md x}$ is a solution to (\ref{eq:problem_norm2}). Then,
(\ref{eq:active_nec_2}) has already been proven as a part of
Lemma~\ref{lemma:active_basic} and we only need to focus on
(\ref{eq:active_nec_1}). The approach follows the same lines as the proof of
Lemma~\ref{lemma:active_basic}: Starting from (\ref{eq:limit_equiv}), which is
necessary and sufficient for vector $\hat{\md x}$ with support $\support$ to be
a solution to (\ref{eq:problem_norm2}), we project the right-hand side onto
an accurately chosen subspace to derive the desired necessary condition.

More specifically, we focus on vectors with support in
$\supportnext\in\Pi(\support)$ and project onto the set of indices $\supportnext
\setminus \support$, for any $\supportnext \in \Pi(\support)$. Then,
(\ref{eq:limit_equiv}) takes the form
$$
\md w_{\supportnext\setminus\support}^\tT[\nabla
L(\hat{\md x})]_{\supportnext\setminus\support} + \mu
\sum_{G\in\partition_{\support}\cap\partition_{\supportnext\setminus\support}}
\norm{\hat{\md x}_{G\cap\support}}_1
\norm{\md w_{G\cap(\supportnext\setminus\support)}}_1 \ge 0.
$$
Since, as already proven, $[\nabla L(\hat{\md x})]_i = 0$ for all $i\in
\support^\tinact$, we can further simplify the last inequality and write
$$
\sum_{G\in\partition_{\support}\cap\partition_{\supportnext\setminus\support}}
\md w_{G\cap(\supportnext\setminus\support)}^\tT[\nabla
L(\hat{\md x})]_{G\cap(\supportnext\setminus\support)} + \mu
\sum_{G\in\partition_{\support}\cap\partition_{\supportnext\setminus\support}}
\norm{\hat{\md x}_{G\cap\support}}_1
\norm{\md w_{G\cap(\supportnext\setminus\support)}}_1 \ge 0.
$$
By recalling that the inequality must hold for all $\md w$ and exploiting the
fact that the index sets $G\in\partition$ are disjoint, we obtain the equivalent
form
$$
\sum_{G\in\partition_{\support}\cap\partition_{\supportnext\setminus\support}}
\inf_{\md w_{G\cap(\supportnext\setminus\support)}}\Bigl\{
\md w_{G\cap(\supportnext\setminus\support)}^\tT[\nabla
L(\hat{\md x})]_{G\cap(\supportnext\setminus\support)} + \mu
\norm{\hat{\md x}_{G\cap\support}}_1
\norm{\md w_{G\cap(\supportnext\setminus\support)}}_1\Bigr\} \ge 0
$$
which implies
$$
\max_{G\in\partition_{\support}\cap\partition_{\supportnext\setminus\support}}
\frac{\Bignorm{[\nabla L(\hat{\md x})]_{G\cap(\supportnext\setminus\support)}}_{\infty}}
{\norm{\hat{\md x}_{G\cap\support}}_1} \le \mu.
$$
Being $\supportnext \in \Pi(\support)$ arbitrarily chosen,
Proposition~\ref{prop:active_necessary} is proven.

\subsection{Proof of Proposition~\ref{prop:active_sufficient}}
Let $\hat{\md u} = \nabla L(\hat{\md x})$ and $\hat{\md u}_\support = \nabla
L_\support(\hat{\md x}_\support)$ be the feasible dual points associated to
$\hat{\md x}$ (and, in turn, to $\hat{\md x}_\support$) as a solution to the
full and restricted problems, respectively. Note that the notation is consistent
since, indeed, $[\nabla L(\hat{\md x})]_j = [\nabla L_\support(\hat{\md x}_\support)]_j$
for all $j\in \support$. However, in general, $[\nabla L(\hat{\md x})]_j \ne 0$
for $j\notin \support$ and, equivalently, $\hat{\md u} \ne \hat{\md
u}_{\{\support\}}$, that is $\hat{\md u}$ may be non-zero outside the support of
$\hat{\md x}$.

As proven in Section~\ref{apdx:preliminary}, the duality gap obtained by
considering $\hat{\md x}$ as a solution to the full problem is given by
(\ref{eq:gap}). Then, by Lemma~\ref{lemma:dual_upper_bound}, this duality gap can
be majorized by
\begin{multline*}
\frac{1}{2\mu}[\ournorm^*(\hat{\md u})]^2
  - \frac{1}{2\mu} [\ournormrest^*(\hat{\md u}_\support)]^2\\
\le \frac{1}{2\mu}\max\Bigl\{0,
  [\ournormrest^*(\hat{\md u}_{\support^\toff})]^2
    - [\ournormrest^*(\hat{\md u}_\support)]^2,
  [\ournormrestinact^*(\hat{\md u}_{\support^\tinact})]^2
    - [\ournormrest^*(\hat{\md u}_\support)]^2\Bigr\}
\end{multline*}

Then, to achieve a duality gap smaller than $\epsilon$, it suffices that
$$
\begin{cases}
  [\ournormrestoff^*(\hat{\md u}_{\support^\toff})]^2
    \le 2\mu\epsilon + [\ournormrest^*(\hat{\md u}_\support)]^2 \\
  [\ournormrestinact^*(\hat{\md u}_{\support^\tinact})]^2
    \le 2\mu\epsilon + [\ournormrest^*(\hat{\md u}_\support)]^2.
\end{cases}
$$

The lemma comes from the definition of the (restricted) dual norm, the fact
that we set $\hat{\md u}=\nabla L(\hat{\md x})$ and (\ref{eq:primal-dual}).

\section{Proof of Theorem \ref{thm:consist_suffic}} \label{apdx:proof_consistency}

The proof capitalizes on the results of \cite{Wainwright09} and
references therein. We summarize next the salient points and refer to those
papers for more specific technical results.

We first point out that a certain $\check{\md x}_n \in \mathbb{R}^p$ is an optimal solution of the problem in (\ref{eq:regression}) if and only if there exists a subgradient $\check{\md u}_n \in \partial \ournorm(\md x)$ at $\md x = \check{\md x}_n$ such that 
\begin{equation} \label{eq:globaloptimality}
\frac{1}{n} \md A_n^\tT \md A_n \left( \check{\md x}_n - \md x_n^\ast \right) - \frac{1}{n} \md A_n^\tT \md w_n + \lambda_n \check{\md u}_n = \md 0
\end{equation}
where we have used the fact that $\md y_n = \md A_n \md x_n^\ast +\md w_n$. Following the steps in \cite{Wainwright09}, we will now construct a sequence of pairs $(\check{\md x}_n,\check{\md u}_n) \in \mathbb{R}^p \times \mathbb{R}^p$ and check under which asymptotic conditions they are solutions to (\ref{eq:globaloptimality}). It is worth remarking that, even though at a first glance it may seem that
the true support $\support=\supp(\md x_n^*)$ is assumed to be known, the
argument below provides a set of equivalent conditions that, if met, ensure that
$\support$ is the correct support and, more importantly, that
$\signsupp(\check{\md x}) = \signsupp(\md x^*)$. Therefore, any efficient method
for solving (\ref{eq:regression}) will converge towards the correct solution
when $\lambda_n$ is properly set.

On the one hand, we consider first a vector $\check{\md x}_n \in \mathbb{R}^p$ that is built as follows. Take $\check{\md x}_\support \in \mathbb{R}^{|\support|}$ as the unique solution to the following \emph{restricted} problem:
\begin{equation}
    \check{\md x}_\support = \arg\min_{\md x_\support \in \mathbb{R}^{|\support|}} \left\{\frac{1}{2n}\norm{\md y_n - \md A_\support \md x_\support}_2^2 + \lambda_n \ournormrest(\md x_\support) \right\}
\end{equation}
and force $\check{\md x}_{\support^\tc}= 0$. On the other hand, we choose $\check{\md u}_n \in \mathbb{R}^p$ by first taking $\check{\md{u}}_\support$ to be any subgradient of $\ournormrest(\md x_\support)$ at the point $\check{\md x}_\support$. Then, $\check{\md{u}}_{\support^\tc}$ is fixed so that the complete vector $\check{\md{u}}_n$ is a solution to the optimality condition of the global problem in (\ref{eq:globaloptimality}). This is possible because we can express this condition as
$$
\frac{1}{n} 
\begin{bmatrix} \md A_\support^\tT \md A_\support & \md A_\support^\tT \md A_{\support^\tc} \\ 
\md A_{\support^\tc}^\tT \md A_\support & \md A_{\support^\tc}^\tT \md A_{\support^\tc} \end{bmatrix} 
\begin{bmatrix}
\check{\md{x}}_\support - \md x_\support^\ast \\ \md 0 
\end{bmatrix}
- \frac{1}{n} \begin{bmatrix} \md A_\support^\tT \\ \md A^\tT_{\support^\tc} \end{bmatrix} \md w_n + \lambda_n \begin{bmatrix}
\check{\md u}_\support \\ \check{\md u}_{\support^\tc}
\end{bmatrix} = \md 0
$$
so that we only need to choose
\begin{equation}
\check{\md u}_{\support^\tc}  =  \md A_{\support^\tc}^\tT 
\left[ \md A_\support \left( \md A_\support^\tT \md A_\support \right)^{-1} \check{\md u}_\support +  \frac{1}{n\lambda_n} \Pi_{\md A_\support^\perp} \md w_n \right]
\end{equation}
where $\Pi_{\md A_\support^\perp} = \md I_p -\md A_\support \left( \md A_\support^\tT \md A_\support \right)^{-1} \md A_\support^\tT $ is the projection matrix onto the space orthogonal to the span of $\md A_\support$. From the global optimality condition we can also write 
\begin{equation} \label{eq:xsuppminusxtrue}
    \check{\md{x}}_\support - \md x_\support^\ast = \left(\frac{1}{n} \md A_\support^\tT \md A_\support \right)^{-1} \left[\frac{1}{n} \md A_\support^\tT \md w_n - \lambda_n \check{\md u}_\support \right].
\end{equation}
Now, following the steps of \cite{Wainwright09}, if we are able to check that $\check{\md u}_n$ is a subgradient of $\ournorm$ at $\check{\md x}_n$ we will have established that $\check{\md x}_n$ is indeed a solution to the original problem. This means that we need to check the dual feasibility condition (from Lemma~\ref{lemma:norm_dual}, with the conditions on $\check{\md u}_{\support}$ following directly from its definition)
$$
\norm{\check{\md u}_{G \cap \support^\tc}}_\infty \leq \frac{\norm{\check{\md x}_{G}}_1}{\ournorm(\check{\md x}_n)},\quad \forall G \in \partition.
$$
Furthermore, if we are able to prove that the inequality above is strict, we will have shown that $\supp(\check{\md x}) = \supp(\md x^*) = \support$ and that the solution is unique. Finally, it is also shown in \cite{Wainwright09} that, if we can check that $\sign(\check{\md u}_\support) = \sign(\md x^\ast_\support)$, we will have established that the original problem has a unique solution $\check{\md x}$ with the correct signed support, that is $\mathbb{S}_\pm(\check{\md x}_n) = \mathbb{S}_\pm(\md x^\ast_n)$.

In summary, we only need to study the condition
\begin{equation} \label{eq:conditionsubg}
    \left\|
\md A_{G \cap \support^\tc}^\tT \md A_\support \left( \md A_\support^\tT \md A_\support \right)^{-1} \check{\md u}_\support +  \frac{1}{n\lambda_n} \md A_{G \cap \support^\tc}^\tT \Pi_{\md A_\support^\perp} \md w_n \right\|_\infty < \frac{\norm{\check{\md x}_{G}}_1}{\ournorm(\check{\md x}_n)},\quad \forall G \in \partition 
\end{equation}
together with, from (\ref{eq:xsuppminusxtrue}) and since $\sign(\check{\md u}_\support) = \sign(\check{\md x}_\support)$,
\begin{equation} \label{eq:conditionSign}
    \sign\left(\check{\md{x}}_\support\right) = \sign\left(\md x_\support^\ast + \left(\frac{1}{n} \md A_\support^\tT \md A_\support \right)^{-1} \left(\frac{1}{n} \md A_\support^\tT \md w_n - \lambda_n \check{\md u}_\support \right)\right).
\end{equation}
In order to establish (\ref{eq:conditionSign}), it is sufficient to establish that
\begin{equation}\label{eq:conditionSign2}
    \norm{\check{\md{x}}_\support - \md x_\support^\ast}_\infty  = 
    \biggnorm{\md A_\support^\# \md w_n - \lambda_n \left(\frac{1}{n} \md A_\support^\tT \md A_\support \right)^{-1}  \check{\md u}_\support }_\infty < |x^\ast|_{\min}
\end{equation}
where we have defined $\md A_\support^\# = \left(  \md A_\support^\tT \md A_\support \right)^{-1}  \md A_\support^\tT $.

Let $\Theta_{n}$ and $\widetilde{\Theta}_{n}$ respectively denote the events $\Theta_{n} =\{\text{(\ref{eq:conditionsubg}) holds true} \}$ and $\widetilde{\Theta}_{n}=\{\text{(\ref{eq:conditionSign2}) holds true} \} $. We will now show that, under the theorem assumptions, we can easily control the probability of the event $\Theta_{n} \cap \widetilde{\Theta}_{n}$ (or, equivalently, the event $\Theta_{n}^\tc \cup \widetilde{\Theta}_{n}^\tc$) for all $n$ large enough. By virtue of the union bound, it is sufficient to find an upper bound on the two probabilities $\Pr(\Theta_{n}^\tc)$ and $\Pr(\widetilde{\Theta}_{n}^\tc)$ separately. 

\subsection{Upper bound on $\Pr(\widetilde{\Theta}_{n}^\tc)$} \label{sec:tildetheta}

Using the triangle inequality in (\ref{eq:conditionSign2}) and the fact that $\ournormrest^\ast(\check{\md u}_\support) = 1$ we can readily observe that
\begin{align*}
    \norm{\check{\md{x}}_\support - \md x_\support^\ast}_\infty & \leq 
    \bignorm{\md A_\support^\# \md w_n}_\infty + \lambda_n 
    \biggnorm{\left(\frac{1}{n} \md A_\support^\tT \md A_\support \right)^{-1}  \check{\md u}_\support }_\infty \\
   & \leq \bignorm{\md A_\support^\# \md w_n}_\infty + \lambda_n 
    \biggnorm{\left(\frac{1}{n} \md A_\support^\tT \md A_\support \right)^{-1}}_{\ournormrest^\ast \to \infty} \leq \bignorm{\md A_\support^\# \md w_n}_\infty + \lambda_n C_\infty.
\end{align*}
Therefore, using \cite{Wainwright09}, we can see that each entry of $\md A_\support^\# \md w_n$ is zero-mean and sub-Gaussian with parameter at most
$$
\frac{\sigma^2}{n} \biggnorm{\Bigl(\frac{1}{n}\md A_\support^\tT \md A_\support\Bigr)^{-1}}_{2\to 2}
\le \frac{\sigma^2}{n C_{\min}}
$$
because of (A2) and, for any fixed $\upsilon \in (0,1)$ independent of $n$, we can write
\begin{align*}
    \Pr \left( \widetilde{\Theta}^\tc_n \right) & \leq \Pr \left( \bignorm{\md A_\support^\# \md w_n}_\infty \geq \upsilon |x^\ast|_{\min} - \lambda_n  C_\infty  \right) \\
    & \leq 2 \exp \left(- \frac{n C_{\min}}{2 \sigma^2} \left( \upsilon |x^\ast|_{\min} - \lambda_n  C_\infty \right)^2 +\log |\support|  \right)
\end{align*}
and therefore $\widetilde{\Theta}_{n}$ holds with an exponentially high probability. 
\subsection{Upper bound on $\Pr(\Theta_{n}^\tc)$}
Observing (\ref{eq:conditionsubg}), we can write
\begin{equation} 
\Pr(\Theta_{n}^\tc) = 
\Pr\left( \max_{G \in \partition} 
\left\|
\md A_{G \cap \support^\tc}^\tT \left[ \md A_\support \left( \md A_\support^\tT \md A_\support \right)^{-1} \check{\md u}_\support +  \frac{1}{n\lambda_n} \Pi_{\md A_\support^\perp} \md w_n \right] \right\|_\infty - \frac{\norm{\check{\md x}_{G}}_1}{\ournorm(\check{\md x}_n)} \geq 0
\right).
\end{equation}
By the triangle inequality and the
definition of $\check{\md u}_\support$ (and, in particular, the fact that
$\ournormrest^*(\check{\md u}_\support) = 1$), we can write
\begin{multline*}
\biggnorm{\md A_{G\cap\support^\tc}^\tT \biggl[\frac{1}{n} \md A_\support
\Bigl(\frac{1}{n}\md A_\support^\tT \md A_\support \Bigr)^{-1} \check{\md u}_\support
+ \frac{1}{n\lambda_n} \md \Pi_{\md A_\support^\perp} \md w_n \biggr]}_\infty \\
\le \biggnorm{\frac{1}{n} \md A_{G\cap\support^\tc}^\tT \md A_\support
\Bigl(\frac{1}{n}\md A_\support^\tT \md A_\support
\Bigr)^{-1}}_{\ournormrest^* \to \infty} + \biggnorm{\frac{1}{n\lambda_n}
\md A_{G\cap\support^\tc}^\tT \md \Pi_{\md A_\support^\perp} \md w_n}_\infty \\
< \frac{\norm{\md x^\ast_G}_1}{\ournorm(\md x_n^\ast)} -\gamma_n + \biggnorm{\frac{1}{n\lambda_n}
\md A_{G\cap\support^\tc}^\tT \md \Pi_{\md A_\support^\perp} \md w_n}_\infty
\end{multline*}
where we have used (\ref{eq:incoherence}). 
This implies that we can write  
\begin{multline*}
\Pr(\Theta_{n}^\tc) \leq \Pr \left(
\max_{G \in \partition} \biggnorm{\frac{1}{n\lambda_n}
\md A_{G\cap\support^\tc}^\tT \md \Pi_{\md A_\support^\perp} \md w_n}_\infty
+ \frac{\norm{{\md x}^\ast_{G}}_1}{\ournorm({\md x}^\ast_n)} - \frac{\norm{\check{\md x}_{G}}_1}{\ournorm(\check{\md x}_n)} > \gamma_n
\right) \\
\leq \Pr \left(
\biggnorm{\frac{1}{n\lambda_n}
\md A_{\support^\tc}^\tT \md \Pi_{\md A_\support^\perp} \md w_n}_\infty > \frac{\gamma_n}{2}
\right) +
\Pr \left(
\max_{G \in \partition} \left| \frac{\norm{{\md x}^\ast_{G}}_1}{\ournorm({\md x}^\ast_n)} - \frac{\norm{\check{\md x}_{G}}_1}{\ournorm(\check{\md x}_n)} \right| > \frac{\gamma_n}{2}
\right) 
\end{multline*}
where we have used the fact that 
$$
\max_{G\in \partition} \biggnorm{\frac{1}{n\lambda_n}
\md A_{G\cap\support^\tc}^\tT \md \Pi_{\md A_\support^\perp} \md w_n}_\infty
= \biggnorm{\frac{1}{n\lambda_n}
\md A_{\support^\tc}^\tT \md \Pi_{\md A_\support^\perp} \md w_n}_\infty.
$$
Now, it is proven in full detail in \cite{Wainwright09} that, under our assumptions on the
noise vector $\md w$, we have
$$
\Pr \biggl(\biggnorm{\frac{1}{n\lambda_n}
\md A_{\support^\tc}^\tT \md \Pi_{\md A_\support^\perp} \md w_n}_\infty >
\frac{\gamma_n}{2}\biggr) \le 2 \exp\Bigl\{-\frac{\lambda_n^2 n \gamma_n^2}{8
\sigma^2} + \log(p-\abs{\support})\Bigr\}.
$$

On the other hand, for any fixed $G \in \partition$, by the triangle inequality,
$$
\left| \norm{\check{\md x}_{G \cap \support}}_1 - \norm{\md x^\ast_{G \cap \support}}_1  \right| \leq \norm{\check{\md x}_{G \cap \support} - \md x^\ast_{G \cap \support}}_1 
$$
and also
$$
\left| \ournormrest \left(\check{\md x}_{ \support} \right) -\ournormrest \left( \md x^\ast_{\support}\right) \right|
\leq \ournormrest \left(\check{\md x}_{ \support} - \md x^\ast_{\support}\right).
$$
Using this, we can immediately express 
\begin{multline*}
\left| \frac{\norm{\check{\md x}_{G \cap \support}}_1}{\ournormrest (\check{\md x}_\support )} - 
\frac{\norm{\md x^\ast_{G \cap \support}}_1}{\ournormrest (\md x^\ast_\support )} \right| \leq \\
\leq
\frac{\norm{\check{\md x}_{G \cap \support}}_1 \left| \ournormrest ({\md x}^\ast_\support ) - \ournormrest (\check{\md x}_\support ) \right|}{\ournormrest (\check{\md x}_\support )\ournormrest ({\md x}^\ast_\support )}
+ \frac{|\norm{\md x^\ast_{G \cap \support}}_1 - \norm{\check{\md x}_{G \cap \support}}_1|}{\ournormrest (\md x^\ast_\support )} \leq \\
\leq \left| \frac{\norm{\check{\md x}_{G \cap \support}}_1}{\ournormrest (\check{\md x}_\support )} - 
\frac{\norm{\md x^\ast_{G \cap \support}}_1}{\ournormrest (\md x^\ast_\support )} \right|
\frac{ \ournormrest ({\md x}^\ast_\support - \check{\md x}_\support ) }{\ournormrest ({\md x}^\ast_\support )} +\\
+\frac{\norm{\md x^\ast_{G \cap \support}}_1}{\ournormrest (\md x^\ast_\support )} 
\frac{\ournormrest ({\md x}^\ast_\support - \check{\md x}_\support ) }{\ournormrest ({\md x}^\ast_\support )} + \frac{\norm{\md x^\ast_{G \cap \support} - \check{\md x}_{G \cap \support}}_1}{\ournormrest (\md x^\ast_\support )}.
\end{multline*}
Next, observe that we can write, for any given $\md x_\support$,
$$
\Phi_\support(\partition) \min_{j \in \support}|x_{n,j}| \leq \ournormrest(\md x_\support) \leq \Phi_\support(\partition) \max_{j \in \support}|x_{n,j}|
$$
where $ \Phi_\support(\partition) $ is defined in (\ref{eq:defPhi}).
Hence, using the upper and lower bounds on the nonzero entries of $\md x^\ast_n$ we are able to refine the above inequality as 
\begin{align*}
\left| \frac{\norm{\check{\md x}_{G \cap \support}}_1}{\ournormrest (\check{\md x}_\support )} - 
\frac{\norm{\md x^\ast_{G \cap \support}}_1}{\ournormrest (\md x^\ast_\support )} \right| & \leq \left| \frac{\norm{\check{\md x}_{G \cap \support}}_1}{\ournormrest (\check{\md x}_\support )} - 
\frac{\norm{\md x^\ast_{G \cap \support}}_1}{\ournormrest (\md x^\ast_\support )} \right|
\frac{\norm{{\md x}^\ast_\support - \check{\md x}_\support}_\infty}{|x^\ast|_{\min}} + \\
& +\frac{|G \cap \support|}{\Phi_\support(\partition)} \left(1 + \frac{|x^\ast|_{\max}}{|x^\ast|_{\min}} \right) \frac{\norm{{\md x}^\ast_\support - \check{\md x}_\support}_\infty}{|x^\ast|_{\min}}.
\end{align*}
Now, let us consider the error term
$$
\Delta_n = \frac{\norm{{\md x}^\ast_\support - \check{\md x}_\support}_\infty}{|x^\ast|_{\min}}
$$ 
and observe that, according to the above characterization of the event $\widetilde{\Theta}_n$, we have $\Delta_n < \upsilon$ for any $0<\upsilon<1$, independently of $n$ and with exponentially high
probability. We can therefore state that 
\begin{multline*}
\Pr \left(
\max_{G \in \partition} \left| \frac{\norm{{\md x}^\ast_{G}}_1}{\ournorm({\md x}^\ast_n)} - \frac{\norm{\check{\md x}_{G}}_1}{\ournorm(\check{\md x}_n)} \right| > \frac{\gamma_n}{2}
\right) \\
\leq
\Pr \left( \left.
\max_{G \in \partition} \left| \frac{\norm{{\md x}^\ast_{G}}_1}{\ournorm({\md x}^\ast_n)} - \frac{\norm{\check{\md x}_{G}}_1}{\ournorm(\check{\md x}_n)} \right| > \frac{\gamma_n}{2}
\right|_{\Delta_n <\upsilon}\right) + \Pr(\Delta_n \geq \upsilon) \\
\leq \Pr \left( \left.
\frac{\upsilon}{1-\upsilon} \frac{\max_{G \in \partition} |G \cap \support|}{\Phi_\support(\partition)} \left(1+ \frac{|x^\ast|_{\max}}{|x^\ast|_{\min}} \right)  > \frac{\gamma_n}{2}
\right|_{\Delta_n < \upsilon} \right) + \Pr( \Delta_n \geq \upsilon).
\end{multline*}
Now, we observe from (\ref{eq:incoherence}) that, since
$$
\ournormrest(\md x^\ast_n) > \min_{G \in \partition} \norm{\md x^\ast_G}_1 \sqrt{|\partition|}
$$
we have
$$
\gamma_n < \min_{G \in \partition} \frac{\norm{\md x_G^\ast}_1}{\ournormrest(\md x_n^\ast)}  < \frac{1}{\sqrt{|\partition|}}
$$
and therefore, since $\Phi_\support(\partition) \le \max_{G \in \partition}
\abs{G \cap \support}\sqrt{\abs{\partition}}$ from (\ref{eq:defPhi}),
$$
\frac{\gamma_n \Phi_\support(\partition)}{2 \max_{G \in \partition} |G \cap \support|} \left( 1 + \frac{|x^\ast|_{\max}}{|x^\ast|_{\min}} \right)^{-1} 
< 
\frac{1}{2}  \left( 1 + \frac{|x^\ast|_{\max}}{|x^\ast|_{\min}} \right)^{-1} 
< 
\frac{1}{4}.
$$
Therefore, if we choose $\upsilon > \frac{4}{5}$
we can guarantee that 
\begin{align*}
\Pr \left(
\max_{G \in \partition} \left| \frac{\norm{{\md x}^\ast_{G}}_1}{\ournorm({\md x}^\ast_n)} - \frac{\norm{\check{\md x}_{G}}_1}{\ournorm(\check{\md x}_n)} \right| > \frac{\gamma_n}{2}
\right) & \leq \Pr(\Delta_n \geq \upsilon) \\
& \leq 2 \exp \left(- \frac{n C_{\min}}{2 \sigma^2} \left( \upsilon |x^\ast|_{\min} - \lambda_n  C_\infty \right)^2 +\log |\support|  \right)
\end{align*}
for all $n$ sufficiently large, where the last inequality follows from Section \ref{sec:tildetheta}. In conclusion, we have shown that 
\begin{align*}
\Pr (\widetilde{\Theta}_n^\tc \cup \Theta_n^\tc ) & \leq 
4 \exp \left(- \frac{n C_{\min}}{2 \sigma^2} \left( \upsilon |x^\ast|_{\min} - \lambda_n  C_\infty \right)^2 +\log |\support|  \right) \\
& + 2 \exp\Bigl\{-\frac{\lambda_n^2 n \gamma_n^2}{8
\sigma^2} + \log(p-\abs{\support})\Bigr\}
\end{align*}
for all $n$ sufficiently large. 
With the choice of $\lambda_n$ in the statement of the theorem, the above probability 
converges to zero at an exponential rate, thus completing the proof.


\begin{thebibliography}{38}
\providecommand{\natexlab}[1]{#1}
\providecommand{\url}[1]{\texttt{#1}}
\expandafter\ifx\csname urlstyle\endcsname\relax
  \providecommand{\doi}[1]{doi: #1}\else
  \providecommand{\doi}{doi: \begingroup \urlstyle{rm}\Url}\fi

\bibitem[Adcock et~al.(2017)Adcock, Hansen, Poon, and Roman]{adcock17}
B.~Adcock, A.~C. Hansen, C.~Poon, and B.~Roman.
\newblock Breaking the coherence barrier: a new theory for compressed sensing.
\newblock \emph{Forum of Mathematics, Sigma}, 5, 2017.

\bibitem[Bach et~al.(2012)Bach, Jenatton, Mairal, and Obozinski]{Bach12}
F.~Bach, R.~Jenatton, J.~Mairal, and G.~Obozinski.
\newblock Optimization with sparsity-inducing penalties.
\newblock \emph{Found. Trends Mach. Learn.}, 4\penalty0 (1):\penalty0 1--106,
  Aug. 2012.

\bibitem[Bastounis and Hansen(2017)]{Bastounis17}
A.~Bastounis and A.~C. Hansen.
\newblock On the absence of uniform recovery in many real-world applications of
  compressed sensing and the restricted isometry property and nullspace
  property in levels.
\newblock \emph{SIAM Journal on Imaging Sciences}, 10\penalty0 (1):\penalty0
  335--371, 2017.

\bibitem[Bayram and Bulek(2017)]{Bayram17}
{\. I}.~Bayram and S.~Bulek.
\newblock A penalty function promoting sparsity within and across groups.
\newblock \emph{{IEEE} Trans. Signal Process.}, 65\penalty0 (16):\penalty0
  4238--4251, Aug. 2017.

\bibitem[Beck and Teboulle(2009)]{Beck09}
A.~Beck and M.~Teboulle.
\newblock A fast iterative shrinkage-thresholding algorithm for linear inverse
  problems.
\newblock \emph{SIAM J. Imaging Sci.}, 2\penalty0 (1):\penalty0 183--202, Jan.
  2009.

\bibitem[Bertsekas(2003)]{Bertsekas_Convex}
D.~P. Bertsekas.
\newblock \emph{Convex Analysis and Optimization}.
\newblock Athena Scientific, Nashua, NH, USA, 2003.
\newblock ISBN 978-1886529458.

\bibitem[Boyd and Vandenberghe(2004)]{Boyd04}
S.~Boyd and L.~Vandenberghe.
\newblock \emph{Convex Optimization}.
\newblock Cambridge University Press, 2004.

\bibitem[Campbell and Allen(2017)]{Cambell17}
F.~Campbell and G.~I. Allen.
\newblock {Within group variable selection through the Exclusive Lasso}.
\newblock \emph{Electronic Journal of Statistics}, 11\penalty0 (2):\penalty0
  4220 -- 4257, 2017.

\bibitem[Cand\`es et~al.(2006)Cand\`es, Romberg, and Tao]{Candes06}
E.~J. Cand\`es, J.~K. Romberg, and T.~Tao.
\newblock Stable signal recovery from incomplete and inaccurate measurements.
\newblock \emph{Communications on Pure and Applied Mathematics}, 59\penalty0
  (8):\penalty0 1207--1223, 2006.

\bibitem[Chandrasekaran et~al.(2010)Chandrasekaran, Recht, Parrilo, and
  Willsky]{Chandra2010}
V.~Chandrasekaran, B.~Recht, P.~A. Parrilo, and A.~S. Willsky.
\newblock The convex geometry of linear inverse problems.
\newblock \emph{Found. Comput. Math.}, 12\penalty0 (6):\penalty0 805--849, Dec.
  2010.

\bibitem[Chen et~al.(2011)Chen, Yuan, Chen, Yan, and Chua]{Chen11}
X.~Chen, X.-T. Yuan, Q.~Chen, S.~Yan, and T.-S. Chua.
\newblock Multi-label visual classification with label exclusive context.
\newblock In \emph{2011 International Conference on Computer Vision}, pages
  834--841, 2011.

\bibitem[Chun and Adcock(2017)]{Chun17}
I.~Y. Chun and B.~Adcock.
\newblock Compressed sensing and parallel acquisition.
\newblock \emph{IEEE Transactions on Information Theory}, 63\penalty0
  (8):\penalty0 4860--4882, 2017.

\bibitem[Combettes and Pesquet(2011)]{Combettes11}
P.~L. Combettes and J.-C. Pesquet.
\newblock Proximal splitting methods in signal processing.
\newblock In H.~H. Bauschke, R.~S. Burachik, P.~L. Combettes, V.~Elser, D.~R.
  Luke, and H.~Wolkowicz, editors, \emph{Fixed-Point Algorithms for Inverse
  Problems in Science and Engineering}, volume~49 of \emph{Springer
  Optimization and Its Applications}. Springer, 2011.

\bibitem[Donoho(2006)]{Donoho06}
D.~Donoho.
\newblock Compressed sensing.
\newblock \emph{IEEE Transactions on Information Theory}, 52\penalty0
  (4):\penalty0 1289--1306, 2006.

\bibitem[Dorsch and Rauhut(2017)]{Dorsch17}
D.~Dorsch and H.~Rauhut.
\newblock Refined analysis of sparse {MIMO} radar.
\newblock \emph{Journal of Fourier Analysis and Applications}, 23\penalty0
  (3):\penalty0 485--529, 2017.

\bibitem[Huang and Zhang(2010)]{Huang10}
J.~Huang and T.~Zhang.
\newblock {The benefit of group sparsity}.
\newblock \emph{The Annals of Statistics}, 38\penalty0 (4):\penalty0 1978 --
  2004, 2010.

\bibitem[Huang et~al.(2011)Huang, Zhang, and Metaxas]{Huang11}
J.~Huang, T.~Zhang, and D.~Metaxas.
\newblock Learning with structured sparsity.
\newblock \emph{J. Mach. Learn. Res.}, 12:\penalty0 3371--3412, Nov. 2011.

\bibitem[Jenatton et~al.(2011{\natexlab{a}})Jenatton, Audibert, and
  Bach]{Jenatton11b}
R.~Jenatton, J.-Y. Audibert, and F.~Bach.
\newblock Structured variable selection with sparsity-inducing norms.
\newblock \emph{J. Mach. Learn. Res.}, 12:\penalty0 2777--2824, July
  2011{\natexlab{a}}.

\bibitem[Jenatton et~al.(2011{\natexlab{b}})Jenatton, Mairal, Obozinski, and
  Bach]{Jenatton11}
R.~Jenatton, J.~Mairal, G.~Obozinski, and F.~Bach.
\newblock Proximal methods for hierarchical sparse coding.
\newblock \emph{J. Mach. Learn. Res.}, 12:\penalty0 2297--2334, July
  2011{\natexlab{b}}.

\bibitem[Kok et~al.(2019)Kok, Choi, Oh, and Choi]{Kok19}
B.~C. Kok, J.~S. Choi, H.~Oh, and J.~Y. Choi.
\newblock Sparse extended redundancy analysis: Variable selection via the
  exclusive lasso.
\newblock \emph{Multivariate Behavioral Research}, pages 1--21, 2019.
\newblock URL \url{https://doi.org/10.1080/00273171.2019.1694477}.

\bibitem[Kong et~al.(2014)Kong, Fujimaki, Liu, Nie, and Ding]{Kong14}
D.~Kong, R.~Fujimaki, J.~Liu, F.~Nie, and C.~Ding.
\newblock Exclusive feature learning on arbitrary structures via
  $\ell_{1,2}$-norm.
\newblock In Z.~Ghahramani, M.~Welling, C.~Cortes, N.~Lawrence, and K.~Q.
  Weinberger, editors, \emph{Advances in Neural Information Processing
  Systems}, volume~27. Curran Associates, Inc., 2014.

\bibitem[Kowalski(2009)]{Kowalski2009}
M.~Kowalski.
\newblock Sparse regression using mixed norms.
\newblock \emph{Applied and Computational Harmonic Analysis}, 27\penalty0
  (3):\penalty0 303--324, 2009.

\bibitem[Li and Adcock(2019)]{Li19}
C.~Li and B.~Adcock.
\newblock Compressed sensing with local structure: Uniform recovery guarantees
  for the sparsity in levels class.
\newblock \emph{Applied and Computational Harmonic Analysis}, 46\penalty0
  (3):\penalty0 453--477, 2019.
\newblock URL
  \url{https://www.sciencedirect.com/science/article/pii/S1063520317300490}.

\bibitem[Lin et~al.(2019)Lin, Sun, Toh, and Yuan]{lin2019dual}
M.~Lin, D.~Sun, K.-C. Toh, and Y.~Yuan.
\newblock A dual {N}ewton based preconditioned proximal point algorithm for
  exclusive lasso models, 2019.

\bibitem[Moreau(1962)]{Moreau62}
J.~J. Moreau.
\newblock Fonctions convexes duales et points proximaux dans un espace
  {H}ilbertien.
\newblock \emph{C. R. Acad. Sci. Paris}, 255:\penalty0 2897--2899, 1962.

\bibitem[Nesterov(1983)]{Nesterov83}
Y.~Nesterov.
\newblock A method for unconstrained convex minimization problem with the rate
  of convergence {$O(1/k^2)$}.
\newblock \emph{Soviet Math. Dokl.}, 27\penalty0 (2):\penalty0 372--376, 1983.

\bibitem[Nesterov(2007)]{Nesterov2007}
Y.~Nesterov.
\newblock Gradient methods for minimizing composite objective functions.
\newblock CORE Discussion Paper 2007/76, Universit\'e Catholique de Louvain
  (UCL), 2007.
\newblock URL
  \url{http://www.optimization-online.org/DB_FILE/2007/09/1784.pdf}.

\bibitem[Obozinski et~al.(2011)Obozinski, Jacob, and Vert]{Obozinski2011}
G.~Obozinski, L.~Jacob, and J.-P. Vert.
\newblock Group lasso with overlaps: the latent group approach.
\newblock Technical Report inria-00628498, Inria, 2011.
\newblock URL \url{https://arxiv.org/abs/1110.0413}.

\bibitem[Oymak et~al.(2013)Oymak, Thrampoulidis, and Hassibi]{Oymak13}
S.~Oymak, C.~Thrampoulidis, and B.~Hassibi.
\newblock The squared-error of generalized lasso: A precise analysis.
\newblock In \emph{2013 51st Annual Allerton Conference on Communication,
  Control, and Computing (Allerton)}, pages 1002--1009, 2013.

\bibitem[Rao et~al.(2016)Rao, Nowak, Cox, and Rogers]{Rao16}
N.~Rao, R.~Nowak, C.~Cox, and T.~Rogers.
\newblock Classification with the sparse group lasso.
\newblock \emph{{IEEE} Trans. Signal Process.}, 64\penalty0 (2):\penalty0
  448--463, Jan. 2016.

\bibitem[Shervashidze and Bach(2015)]{Shervashidze15}
N.~Shervashidze and F.~Bach.
\newblock Learning the structure for structured sparsity.
\newblock \emph{{IEEE} Trans. Signal Process.}, 63\penalty0 (18):\penalty0
  4894--4902, Sept. 2015.

\bibitem[Tibshirani(1996)]{Tibshirani96}
R.~Tibshirani.
\newblock Regression shrinkage and selection via the lasso.
\newblock \emph{Journal of the Royal Statistical Society. Series B
  (Methodological)}, 58\penalty0 (1):\penalty0 267--288, 1996.

\bibitem[Villa et~al.(2014)Villa, Rosasco, Mosci, and Verri]{Villa2014}
S.~Villa, L.~Rosasco, S.~Mosci, and A.~Verri.
\newblock Proximal methods for the latent group lasso penalty.
\newblock \emph{Comput. Optim. Appl.}, 50\penalty0 (2):\penalty0 381--487, June
  2014.

\bibitem[Wainwright(2009)]{Wainwright09}
M.~J. Wainwright.
\newblock Sharp thresholds for high-dimensional and noisy sparsity recovery
  using {$\ell_1$}-constrained quadratic programming ({L}asso).
\newblock \emph{{IEEE} Trans. Inf. Theory}, 55\penalty0 (5):\penalty0
  2183--2202, May 2009.

\bibitem[Yamada et~al.(2017)Yamada, Koh, Iwata, Shawe-Taylor, and
  Kaski]{yamada17a}
M.~Yamada, T.~Koh, T.~Iwata, J.~Shawe-Taylor, and S.~Kaski.
\newblock {Localized Lasso for High-Dimensional Regression}.
\newblock In A.~Singh and J.~Zhu, editors, \emph{Proceedings of the 20th
  International Conference on Artificial Intelligence and Statistics},
  volume~54 of \emph{Proceedings of Machine Learning Research}, pages 325--333,
  Fort Lauderdale, FL, USA, 20--22 Apr 2017. PMLR.
\newblock URL \url{http://proceedings.mlr.press/v54/yamada17a.html}.

\bibitem[Yuan and Lin(2006)]{Yuan2006}
M.~Yuan and Y.~Lin.
\newblock Model selection and estimation in regression with grouped variables.
\newblock \emph{J. R. Statist. Soc. Ser. B}, 68\penalty0 (1):\penalty0 49--67,
  2006.

\bibitem[Zhang et~al.(2016)Zhang, Ghanem, Liu, Xu, and Ahuja]{Zhang16}
T.~Zhang, B.~Ghanem, S.~Liu, C.~Xu, and N.~Ahuja.
\newblock Robust visual tracking via exclusive context modeling.
\newblock \emph{IEEE Transactions on Cybernetics}, 46\penalty0 (1):\penalty0
  51--63, 2016.

\bibitem[Zhou et~al.(2010)Zhou, Jin, and Hoi]{zhou10a}
Y.~Zhou, R.~Jin, and S.~C. Hoi.
\newblock Exclusive lasso for multi-task feature selection.
\newblock In Y.~W. Teh and M.~Titterington, editors, \emph{Proceedings of the
  Thirteenth International Conference on Artificial Intelligence and
  Statistics}, volume~9 of \emph{Proceedings of Machine Learning Research},
  pages 988--995, Chia Laguna Resort, Sardinia, Italy, 13--15 May 2010. PMLR.
\newblock URL \url{http://proceedings.mlr.press/v9/zhou10a.html}.

\end{thebibliography}
\end{document}